\documentclass{article}

\usepackage{microtype}
\usepackage{graphicx}
\usepackage{subfigure}
\usepackage{booktabs} 

\usepackage{hyperref}

\usepackage[accepted]{icml2023}

\usepackage{amsmath}
\usepackage{amssymb}
\usepackage{mathtools}
\usepackage{amsthm}

\usepackage[capitalize,noabbrev]{cleveref}

\theoremstyle{plain}
\newtheorem{Theorem}{Theorem}[section]

\newtheorem{lemma}[Theorem]{Lemma}

\theoremstyle{definition}

\theoremstyle{remark}

\newcommand\C{\mathbb{C}}
\newcommand\R{\mathbb{R}}

\usepackage[textsize=tiny]{todonotes}

\icmltitlerunning{Analyzing the Neural Tangent Kernel of Periodically
Activated Coordinate Networks}

\begin{document}

\twocolumn[
\icmltitle{Analyzing the Neural Tangent Kernel of Periodically
Activated Coordinate Networks}



\icmlsetsymbol{equal}{*}

\begin{icmlauthorlist}
\icmlauthor{Hemanth Saratchandran}{inst}
\icmlauthor{Shin-Fang Chng}{inst}
\icmlauthor{Simon Lucey}{inst}
\end{icmlauthorlist}

\icmlaffiliation{inst}{Australian Institute of Machine Learning, University of Adelaide, Adelaide SA, Australia}

\icmlcorrespondingauthor{Hemanth Saratchandran}{hemanth.saratchandran@adelaide.edu.au}

\icmlkeywords{Neural Tangent Kernel, Minimum Eigenvalue, Periodically activated coordinate networks}

\vskip 0.3in
]



\printAffiliationsAndNotice{}  

\begin{abstract}

Recently, neural networks utilizing periodic activation functions have been proven to demonstrate superior performance in vision tasks compared to traditional ReLU-activated networks. However, there is still a limited understanding of the underlying reasons for this improved performance. In this paper, we aim to address this gap by providing a theoretical understanding of periodically activated networks through an analysis of their Neural Tangent Kernel (NTK).
We derive bounds on the minimum eigenvalue of their NTK in the finite width setting, using a fairly general network architecture which requires only one wide layer that grows at least linearly with the number of data samples. Our findings indicate that periodically activated networks are \textit{notably more well-behaved}, from the NTK perspective, than ReLU activated networks.
Additionally, we give an application to the memorization capacity of such networks and verify our theoretical predictions empirically. Our study offers a deeper understanding of the properties of periodically activated neural networks and their potential in the field of deep learning.


\end{abstract}

\section{Introduction}\label{intro}
Implicit neural signal representations, also known as coordinate networks, have proven to be an efficient method for modeling 
multi-dimensional continuous signals, achieving state-of-the-art results in several vision applications \cite{skorokhodov2021adversarial, chen2021learning, sitzmann2020implicit, mildenhall2021nerf, li20223d}. 
The recent work \cite{sitzmann2020implicit} has introduced a periodically activated coordinate network that has demonstrated superior performance in image synthesis, signed distance fields, and geometry, compared to conventional coordinate networks trained with sigmoid, Tanh and ReLU activations. 
Despite the promising results, the theoretical understanding of periodically activated coordinate networks remain an uncharted area of research. This paper aims to begin a theoretical analysis of such networks through the lens of the Neural Tangent Kernel (NTK).

We consider a neural network with $L$ layers, where the number of neurons in each layer are represented by 
$\{n_1,\ldots,n_L\}$. The feature maps,
$f_k : \R^{n_0} \rightarrow \R^{n_k}$ for the network are defined for each input $x \in \R^{n_0}$ by
\begin{equation}\label{network_defn}
    f_k(x) =
    \begin{cases}
    W_L^Tf_{L-1}, & k = L \\
    \phi(W_k^Tf_{k-1}), & k = [L-1] \\
    x, & k = 0
    \end{cases}
\end{equation}
where $n_0$ denotes the input dimension, $W_k \in \R^{n_{k-1}\times n_k}$, 
$\phi(x) = \cos(sx)$ is the activation function, with $s$ is a fixed frequency parameter, and the notation
$[m] = \{1,\ldots,m\}$. We assume the output dimension is fixed at $n_L = 1$, so that
$W_L \in \R^{n_{L-1}}$. We use the standard vectorization notation for parameters and write $\theta = [vec(W_1),\ldots,vec(W_L)] \in \R^p$, where 
$p = \sum_{k=1}^Ln_{k-1}\cdot n_k$. Let $(x_1,\ldots,x_N)$ be $N$ samples in 
$\R^{n_0}$, and let $F_L(\theta) = [f_L(x_1),\ldots,f_L(x_N)]^T \in 
\R^N$. Let $J \in \R^{N \times p}$ denote the Jacobian of $F_L$ with respect to the weights
$\theta = (vec(W_1),\ldots,vec(W_L))$. The empirical Neural Tangent Kernel, $K_L$,  is defined by
\begin{equation}
    K_L = JJ^T = \sum_{k=1}^L\left(\frac{\partial F_L}{\partial vec(W_k)} \right)
    \left(\frac{\partial F_L}{\partial vec(W_k)} \right)^T.
\end{equation}
Several works \cite{du2018gradient, allen2019convergence, oymak2020toward, song2019quadratic, nguyen2020global, zou2019improved} have established connections between the spectrum of the empirical NTK matrix and the training of neural networks. A key insight on this front is that when a neural network with feature map $F_L$ is trained with Mean Squared Error (MSE) loss, 
$\mathcal{L}(\theta) = \frac{1}{2}\vert\vert F_L - Y\vert\vert_2^2$, where $Y$ are the training labels, then it can be shown that
\begin{equation}\label{ntk_loss}
    \vert\vert \nabla\mathcal{L}(\theta)\vert\vert_2^2 \geq 
    2\lambda_{\min}\left( K_L \right)\mathcal{L}(\theta).
\end{equation}
Equation \eqref{ntk_loss} shows that if the NTK matrix $K_L$ has a spectral gap, meaning its minimum eigenvalue is bounded away from zero, at initialization and this condition persists throughout training, then minimizing the gradient will drive the loss to zero, resulting in convergence to a global minimum. 
This approach has been utilized in several works to prove global convergence of 
gradient decent for various network architectures.
For example, \cite{allen2019convergence, du2019gradient, zou2019improved} have studied global convergence for deep networks with polynomial width, and \cite{nguyen2020global} has
studied global convergence of deep nets with one wide layer. These works heavily rely on the empirical NTK matrix.
Furthermore, \cite{montanari2022interpolation, nguyen2021tight} have applied minimum eigenvalue bounds to prove memorization capacity of ReLU-activated networks, and \cite{arora2019fine, montanari2022interpolation} has used it to 
prove generalization bounds. These works highlight the importance of understanding the minimum eigenvalue of the NTK matrix. However, all these works only study the case of traditional activation functions such as ReLU. To the best of our knowledge, the study of the NTK matrix of non-traditional activation functions, such as a periodic activation function, has not yet been conducted. 

In this paper, we aim to bridge this gap by providing a theoretical understanding of periodically activated networks through an analysis of their NTK using random matrix methods. We summarize our main contributions as follow:
\begin{enumerate}
    \item[1.] We prove lower and upper bounds on the minimum eigenvalue of the empirical NTK matrix for cosine activated coordinate networks.
    We show that if the network has one wide layer with width $n_k$, positioned anywhere between the input and output, that is linear in the number of training samples $N$, then the minimum eigenvalue of the empirical NTK scales according to $\Theta(n_k^{3/2})$.
    Following the approach taken in \cite{nguyen2020global, nguyen2021tight} and applying \textit{necessary modifications}, we establish suitable bounds on the minimum singular value of the feature matrices associated to the network. In contrast to the result of \cite{nguyen2021tight} which shows that a ReLU-activated network has a scaling of $\Theta(n_k)$, our results reveal that for a cosine activated network, the minimum eigenvalue of the empirical NTK is 
    \textit{far better conditioned} at initialization, resulting in a larger spectral gap. 
    This explains the superior performance of cosine activated networks when trained with gradient descent methods, as previously observed empirically in \cite{sitzmann2020implicit}.


    \item[2.] We apply the obtained lower bound on the eigenvalue of the empirical NTK matrix to prove a memorization theorem for cosine activated coordinate networks, which can be compared to a similar theorem established in \cite{nguyen2021tight} for ReLU networks.


    \item[3.] Finally, we empirically verify the scaling of the minimum eigenvalue of the empirical NTK matrix for a cosine activated network, which confirms the main theorem of our paper. Additionally, we compare these results to those obtained from a parallel set of experiments carried out on a ReLU-activated network, and further solidify our findings on the improved conditioning of cosine activated networks at initialization.
\end{enumerate}

\section{Notation and Assumptions}\label{notations}


In this section, we outline the notation and assumptions that will be used throughout the paper.

\textbf{Notation.} We will fix a depth $L$ neural network $f_L$, defined by \eqref{network_defn}. The data samples will be denoted by 
$X = [x_1,\ldots,x_N]^T \in \R^{N\times n_0}$, where $N$ is the number of samples and $n_0$ is the dimension of input features. The output of the network at layer $k$ will be denoted by $f_k$ and the feature matrix at layer $k$ 
by $F_k = [f_l(x_1),\ldots,f_l(x_N)]^T \in \R^{N \times n_k}$, where $n_k$ is the dimension of features at layer $k$. When $k = 0$, the feature matrix is simply the input data matrix $X$.
We define
$\Sigma_k = D([\phi'(g_{k,j}(x))]_{j=1}^{n_k})$, where $D$ denotes diagonal matrix and 
$g_{k,j}(x)$ denotes the pre-activation neuron i.e. the output of layer k before the activation function is applied.
Note that $\Sigma_k$ is then an $n_k \times n_k$ diagonal matrix. 
We will use standard complexity notations, $\Omega(\cdot)$, 
$\mathcal{O}(\cdot)$, $o(\cdot)$, $\Theta(\cdot)$ throughout the paper, which are all to be understood in the asymptotic regime, where $N, n_0, n_1,\ldots,n_{L-1}$ are sufficiently large. 
Furthermore, we will use the notation "w.p." to denote \textit{with probability}, and "w.h.p." to denote \textit{with high probability} throughout the paper.


\textbf{Network weights distribution.} We will analyze the properties of the network when weights are randomly initialized. Specifically, we assume that all weights are independently and identically distributed (i.i.d) according to a Gaussian distribution, $(W_k)_{i,j} \sim_{i.i.d} \mathcal{N}(0, \beta_k^2)$, where we assume
$\beta_k \leq 1$ for all $1 \leq k \leq L-1$.

\textbf{Data distribution.} We will assume the data samples 
$\{x_i\}_{i=1}^N$ are i.i.d from a fixed distribution denoted by $\mathcal{P}$. The measure associated to $\mathcal{P}$ will be denoted by $d\mathcal{P}$. We will work under the following assumptions, which are standard in the literature.
\begin{enumerate}
    \item[{A1.}]  $\int_{\R^{n_0}}\vert\vert x\vert\vert_2d\mathcal{P} = 
    \Theta(\sqrt{n_0})$.

    \item[A2.]  $\int_{\R^{n_0}}\vert\vert x\vert\vert_2^2d\mathcal{P} = 
    \Theta(n_0)$.
\end{enumerate}
We will also assume the data distribution satisfies Lipschitz concentration. 
\begin{enumerate}
    \item[A3.] For every Lipschitz continuous function 
    $f : \R^{n_0} \rightarrow \R$, there exists an absolute constant $C > 0$ such that, for any $t > 0$, we have
    \begin{equation*}
        \mathbf{P}\left(
            \bigg{\vert}f(x) - \int f(x')d\mathcal{P}\bigg{\vert}
        \right) \leq 2e^{-ct^2/\vert\vert f\vert\vert_{Lip}^2}.
    \end{equation*}
\end{enumerate}
Note that there are several distributions satisfying assumption A3 such as standard Gaussian distributions, and uniform distributions on spheres and cubes, see \cite{vershynin2018high}.

The final assumption we will be making is a Lipschitz constant assumption on the network.
\begin{enumerate}
  \item[A4.] The Lipschitz constant of layer $k$ must satisfy the following bound
    \begin{align*}
       &\hspace{-0.5cm}\vert\vert f_k\vert\vert^2_{Lip} \leq \\
        &\hspace{-0.5cm}\mathcal{O}\left(
        \frac{s^k\beta_kn_k}{\min_{l\in[0,k]}n_l}\left( 
        \prod_{l=1}^{k-1}\sqrt{\beta_l}\sqrt{n_l}
        \right)
        \left(
            \prod_{l=1}^{k-1}Log(n_l)
        \right)
       \right).
    \end{align*}\label{assumption:A4}
\end{enumerate}
In the experiments in Section \ref{exps}, we show that the empirical Lipschitz constant satisfies
such a bound for cosine activated networks and in fact such a bound is extremely pessimistic.

\section{Main Result}\label{sec_main_result}




In this section, we establish bounds on the minimum eigenvalue of the empirical NTK associated to a cosine activated neural network. Our results are valid for a fairly general architecture, requiring only the presence of one wide hidden layer positioned anywhere between the input and output layers. Furthermore, the theorem shows that it is sufficient to have one wide layer with width linear in and greater than the number of training samples $N$.

The result of this paper applies to a wide range of network architectures, where any subset of layers can be chosen to be wide. This is significant as it only requires one wide layer, and its position is not fixed. While similar results for ReLU networks have been established in~\cite{nguyen2020global, nguyen2021tight}, this is the first time such results have been established for cosine activated networks. This expands the understanding of the behaviors of the minimum eigenvalue of NTK matrix in non traditional activations.

An interesting aspect of the theorem is that in the case of a single wide layer of width $n_k$, the scaling of the minimum eigenvalue of the NTK follows a term of the form $\Theta(n_k^{3/2})$. This is in contrast to the case of a ReLU-activated neural network, where the scaling is dominated by a linear term 
$\Theta(n_k)$, as established in \cite{nguyen2021tight}. This implies that the empirical NTK of a cosine activated neural network will generally have a larger spectral gap at initialization than a ReLU-activated network. We confirm this through experiments in Section \ref{exps}. These results suggest that periodically activated networks are generally better conditioned at initialization for training with gradient decent, an observation that has already been empirically established in \cite{sitzmann2020implicit}.

\begin{Theorem}\label{main_result_ntk}
Let $f_L$ denote a depth $L$ neural network with $\phi(x) = cos(sx)$ as the activation, where $s > 0$ is a fixed frequency parameter, satisfying the network assumptions in Section \ref{notations}.
Let $\{x_i\}_{i=1}^N$ denote a set of i.i.d training data points sampled from the distribution $\mathcal{P}$, which satisfies the data assumptions in Section \ref{notations}.
Let 
$a_k = 1$ if the following conditions holds
\begin{align*}
    n_k &\geq N \\
    \prod_{l=1}^{k-1}Log(n_l) &= o\left(\min_{l \in [0,k]}n_l\right)
\end{align*}
and $0$ otherwise. Then
\begin{align*}
    &\lambda_{\min}\left(K_L\right) \\
    &\geq \sum_{k=1}^{L-1}a_k\Omega\Bigg{(}
    s^2(1 - e^{-\beta_{k+1}^2s^2})\beta_k^{3/2}n_k^{3/2} 
    \left(
    \prod_{l=1}^{k-1}\beta_ln_l
    \right) \\
    &\hspace{2cm}\beta_{k+1}n_{k+1}
    \left(
    \prod_{k+2}^{L-1}\sqrt{\beta_l}\sqrt{n_l}
    \right)\beta_Ln_L\Bigg{)} \\
    &\hspace{0.5cm} +
    \lambda_{min}(XX^T) \cdot \\
    &\hspace{1cm}\Omega\left(
    s^2(1 - e^{-\beta_1^2s^2})\beta_1n_1\left(\prod_{l=2}^{L-1} 
    \sqrt{\beta_l}\sqrt{n_l}\right)\beta_Ln_L
    \right)
\end{align*}
w.p. at least
\begin{align*}
    & 1 - \sum_{k=1}^{L-1}(N^2-N)\exp\left(-
\Omega\left(
\frac{\min_{l \in [0,k]}n_l}{s^kN^2\prod_{l=1}^{k-1}Log(n_l)}
\right)
\right) \\
&\hspace{1cm}- N\sum_{l=0}^k2\exp(-\Omega(n_l))
\end{align*}
over $(W_l)_{l=1}^L$ and the data.

Furthermore, we have 
\begin{align*}
    &\lambda_{\min}\left(K_L\right) \\
    &\leq \sum_{k=1}^L\mathcal{O}\Bigg{(}
    s^3(1-e^{\beta_{k+1}^2s^2})\beta_k^{3/2}n_k^{3/2}\sqrt{n_0}
    \left(
    \prod_{l=1}^{k-1}\beta_ln_l 
    \right) \\
    &\hspace{2cm}(\beta_{k+1}n_{k+1})
    \left(
    \prod_{l=k+2}^{L-1}\sqrt{\beta_l}\sqrt{n_l}
    \right)\beta_Ln_L
    \Bigg{)}
\end{align*}
w.p. at least
\begin{align*}
    1 - \sum_{l=1}^k2\exp(-\Omega(sn_l)) - 2\exp(-\Omega(\sqrt{n_0}))
\end{align*}
over $(W_l)_{l=1}^L$ and the data.
\end{Theorem}


    

\section{Proof of Theorem \ref{main_result_ntk}}


In this section, we outline the steps in the proof of Theorem \ref{main_result_ntk}. Our approach builds upon the techniques used in
\cite{nguyen2021tight}.

Recall that the empirical NTK as
\begin{equation*}
    K_L = JJ^T = \sum_{l=1}^k\left(\frac{\partial F_L}{\partial vec(W_l)}\right)
    \left(\frac{\partial F_L}{\partial vec(W_l)}\right)^T.
\end{equation*}

Using the chain rule, we can express the NTK in terms of the feature matrices as:

\begin{equation}\label{ntk_decomp}
    JJ^T = \sum_{k=0}^{L-1}F_lF_l^T\circ G_{k+1}G_{k+1}^T, 
\end{equation}
where $G_k \in \R^{N\times n_k}$ with $i^{th}$ row given by
\begin{equation*}
    (G_k)_{i:} =
    \begin{cases}
    \frac{1_N}{\sqrt{N}}, & \hspace{-0.2cm} k = L \\
    \Sigma_{L-1}(x_i)W_L, & \hspace{-0.2cm} k = L-1 \\
    \Sigma_k(x_i)\left(\prod_{j=k+1}^{L-1}\hspace{-0.1cm}W_j\Sigma_j(x_i)
    \right)
    \hspace{-0.1cm}W_L, & 
    \hspace{-0.25cm} k \in [L-2]
    \end{cases}
\end{equation*}

By Weyl's inequality, we obtain
\begin{equation}\label{weyl_ntk_decomp}
    \lambda_{\min}(JJ^T) \geq \sum_{k=0}^{L-1}\lambda_{\min}(F_kF_k^T \circ 
    G_{k+1}G_{k+1}^T).
\end{equation}
Each term in the sum on the left hand side of \eqref{weyl_ntk_decomp} can be further bounded by Schur's Theorem \cite{schur1911bemerkungen, horn1994topics} to give
\begin{align}
    &\lambda_{\min}(F_kF_k^T \circ G_{k+1}G_{k+1}^T) \nonumber\\
    &\geq 
\lambda_{\min}(F_kF_k^T)\min_{i \in [N]}\vert\vert (G_{k+1})_{i:}\vert\vert_2^2.
\label{schur}
\end{align}
\eqref{weyl_ntk_decomp} and \eqref{schur} imply
\begin{align}\label{ntk_lower_decomp}
    \lambda_{\min}(JJ^T) \geq  \sum_{k=0}^{L-1}
    \lambda_{\min}(F_kF_k^T)\min_{i \in [N]}\vert\vert (G_{k+1})_{i:}\vert\vert_2^2.
\end{align}

The strategy of the proof is to obtain bounds on the terms 
$\lambda_{\min}(F_kF_k^T)$ and $\vert\vert (G_{k+1})_{i:}\vert\vert_2^2$ separately, and then combine them together to obtain a bound on $JJ^T$. Estimating the term
$\lambda_{\min}(F_kF_k^T)$ is done in Section \ref{sec_feature}. The following 
lemma shows how to estimate the quantity $\vert\vert (G_{k+1})_{i:}\vert\vert_2^2$.

\begin{lemma}\label{G-bound}
Fix $k \in [L-2]$ and let $x \sim \mathcal{P}$. Then 
\begin{align*}
&\bigg{\vert}\bigg{\vert}
\Sigma_k(x)\bigg{(}\prod_{l=k+1}^{L-1}W_l\Sigma_l(x)\bigg{)}W_L
\bigg{\vert}\bigg{\vert}^2_2 \\
&=
    \Theta\bigg{(}
s^2(1 - e^{-\beta_k^2s^2})\sqrt{n_0}\beta_kn_k\beta_{L}n_L
\prod_{l=1, l\neq k}^{L-1}
\sqrt{\beta_l}\sqrt{n_l}
\bigg{)}
\end{align*}
w.p. at least
\begin{align*}
    1 - \sum_{l=0}^{L-1}2\exp(-\Omega(n_l)).
\end{align*}
\end{lemma}

The proof of lemma \ref{G-bound} is given in Section \ref{sec_G-bound_proof}.

\begin{proof}[\textbf{Proof of Theorem \ref{main_result_ntk}}]
The lower bound of Theorem \ref{main_result_ntk} follows by applying Theorem \ref{sing_val_feat} and lemma \ref{G-bound} to obtain lower bounds on the terms in the sum on the right hand side of \eqref{ntk_lower_decomp}.

    Using the variational characterization of eigenvalues we have
    \begin{equation*}
        \lambda_{min}(JJ^T) \leq \frac{\langle JJ^Tv, v\rangle}
        {\langle v, v\rangle} \text{ for all } v \in \C^N-\{0\}.
    \end{equation*}
Taking $v = e_i$ to be the ith-standard basis vector in $\C^N-\{0\}$, we obtain
\begin{equation}\label{variational_upbound}
    \lambda_{min}(JJ^T) \leq (JJ^T)_{i,i} \leq 
    \sum_{k=0}^{L-1}\vert\vert(F_k)_{i:}\vert\vert_2^2\vert\vert 
    (G_{k+1})_{i:}\vert\vert^2_2.
\end{equation}
The second term on the right hand side of \eqref{variational_upbound} can be bounded 
by lemma \ref{G-bound}. To bound the term $\vert\vert(F_l)_{i:}\vert\vert_2^2$, we simply observe that $(F_l)_{i:} = f_l(x_i)$ and then by applying lemma \ref{lemmac1_cos}
we have that 
\begin{equation*}
    \vert\vert f_k(x_i)\vert\vert^2_2 = 
    \Theta\bigg{(}s\cdot \sqrt{n_0}\beta_k n_k
	\prod_{l=1}^{k-1}\sqrt{\beta_l}\sqrt{n_l}\bigg{)}
\end{equation*}
w.p. at least
\begin{equation*}
    1 - \sum_{l=1}^k2\exp(-\Omega(sn_l)) - 2\exp(-\Omega(\sqrt{n_0})).
\end{equation*}
By plugging this bound into \eqref{variational_upbound} along with the bound from 
lemma \ref{G-bound} proves the upper bound of Theorem \ref{main_result_ntk}.
\end{proof}

\section{Implication for the Memorization Capability of Network}\label{sec_mem}
Theorem \ref{main_result_ntk} has a significant implication for memorization capability of periodically activated networks. In particular, our results indicate that such networks can fit distinct data point arbitrarily closely, regardless of the associated label values. This result generalizes the memorization result from 
\cite{nguyen2021tight} to the case of periodically activated networks. To the best of our knowledge, this is the first time a memorization capacity theorem has been established for periodically activated networks.
Our proof builds upon the techniques used in \cite{nguyen2021tight}, with the \textit{necessary} adjustments for the periodic case.

\begin{Theorem}\label{mem_cos}

Let $f_L$ denote a neural network of depth $L$ with activation function $\phi(x) = cos(sx)$, where $s > 0$ is a fixed frequency parameter. 
Let $\{x_i\}_{i=1}^N$ denotes a set of i.i.d training data points sampled from distribution $\mathcal{P}$. Assume there exists a $k \in [L-1]$ such that 
\begin{align*}
    n_k &\geq N \\
    \prod_{l=1}^{k-1}Log(n_l) &= o(\min_{l \in [0,k]}n_l)
\end{align*}

Let $Y \in \R^N$ be a given vector (to be thought of as a target). Then for all $\epsilon > 0$, there exists
a parameter $\theta \in \R^p$ such that
\begin{equation*}
    \vert\vert F_L(\theta) - Y\vert\vert_2 < \epsilon
\end{equation*}
w.p. at least 
\begin{align*}
    & 1 - N(N-1)\exp\left(-
\Omega\left(
\frac{\min_{l \in [0,k]}n_l}{s^kN^2\prod_{l=1}^{k-1}Log(n_l)}
\right)
\right) \\
&- N\sum_{l=0}^k2\exp(-\Omega(n_l)).
\end{align*}
\end{Theorem}

\begin{proof}
Let $F_L$ denote the output feature map. The dimension of the parameter space is given by $p = \sum_{l=1}^Ln_l\times n_{l-1}$. As we have fixed the training samples, the neural output feature map can be viewed as a map
\begin{equation*}
F_L : \R^p \rightarrow \R^N.
\end{equation*}
Taking its derivative with respect to a fixed $\theta \in \R^p$, 
we get a linear map
\begin{equation*}
\frac{\partial{F_L}}{\partial{\theta}} : \R^p \rightarrow \R^N.
\end{equation*}

Theorem \ref{main_result_ntk} implies that picking weights according to a Gaussian distribution, satisfying the assumptions in Section \ref{notations}, the map
$\frac{\partial{F_L}}{\partial{\theta}}$ will be a full rank linear map 
w.p. at least
\begin{align*}
    & 1 - N(N-1)\exp\left(-
\Omega\left(
\frac{\min_{l \in [0,k]}n_l}{s^kN^2\prod_{l=1}^{k-1}Log(n_l)}
\right)
\right) \\
&\hspace{0.5cm}- N\sum_{l=1}^k2\exp(-\Omega(n_l)).
\end{align*}

We define 
\begin{equation}\label{rank_set}
\mathcal{R} = \{ \theta \in \R^p : Rank(J(\theta)) = N\}.
\end{equation}
Since the measure associated to a Gaussian probability distribution is absolutely continuous with respect to the Lebesgue measure, we see that 
$\mu(\mathcal{R}) > 0$ w.p. at least
\begin{align*}
    & 1 - N(N-1)\exp\left(-
\Omega\left(
\frac{\min_{l \in [0,k]}n_l}{s^k\prod_{l=1}^{k-1}Log(n_l)}
\right)
\right) \\
&\hspace{0.5cm}- N\sum_{l=1}^k2\exp(-\Omega(n_l))
\end{align*}
over the training data, where 
$\mu$ denote the Lebesgue measure.

This implies $\mathcal{R}$ is not a $\mu$-null set and in particular is not empty. Therefore, pick $\theta_0 \in \mathcal{R}$ and observe that since 
$\frac{\partial{F_L}}{\partial{\theta}}\big{\vert}_{\theta_0}$ is full rank, there exists a $\widetilde{\theta} \in \R^p$ such that
\begin{equation}\label{surj_jac}
	\left(
	\frac{\partial{F_L}}{\partial{\theta}}\bigg{\vert}_{\theta_0}
	\right)\left( \widetilde{\theta} \right) = Y.
\end{equation}

Since our feature maps are defined by 
\begin{equation*}
    F_L(\theta) =  [f_L(x_1),\ldots,f_L(x_N)]^T \in \R^N,
\end{equation*}
on writing 
$Y = [y_1,\ldots,y_N]^T \in \R^N$,
it follows that 
\eqref{surj_jac} implies that for $i \in [N]$ we have that
\begin{equation}\label{difference_quot}
y_i = \bigg{\langle} \frac{\partial f_L(\theta, x_i)}{\partial \theta}\bigg
{\vert}_{\theta_0}, \widetilde{\theta}\bigg{\rangle} =
\lim_{h\rightarrow 0}\frac{f_{L}(\theta_0 + 
h\widetilde{\theta}, x_i)}{h}.
\end{equation}

Define $g_h(x_i) = \frac{f_{L}(\theta_0 + 
h\widetilde{\theta}, x_i)}{h}$. Then observe that given $\epsilon > 0$, there exists $0 < h_i << 1$, such that $h_i = h_i(\epsilon)$, with
\begin{equation*}
\vert g_{h_i}(x_i) - y_i\vert^2 < \frac{\epsilon^2}{N}
\end{equation*}
for each $i \in [N]$. In particular, taking $h = \min_{i \in [N]}\{h_i\}$ we have that
\begin{equation*}
\sqrt{\sum_{i=1}^N \vert g_{h}(x_i) - y_i\vert^2} < \epsilon.
\end{equation*}
Finally observe that by formula \eqref{difference_quot}
the function $g_{h}(x_i)$ can be implemented by a neural network with depth $L$ with widths $\{2n_1,\ldots,2n_L\}$. The result then follows.
\end{proof}

\section{Minimum Singular Value of the Feature Matrix}\label{sec_feature}


In this section, we derive bounds on the minimum singular value of the feature matrices associated to the network. We consider a set of i.i.d data samples 
$\{x_i\}_{i=1}^N$, drawn from distribution $\mathcal{P}$, which is assumed to satisfy assumptions A1-A3. The 
feature matrix at layer $k$ is given by 
$F_k = [f_k(x_1),\ldots.f_k(x_N)]^T \in \R^{N\times n_k}$. For the following theorem, we assume assumption A4.

\begin{Theorem}\label{sing_val_feat}
Let $f_L$ denote a neural network of depth $L$ with activation function $\phi(x) = cos(sx)$, where $s > 0$ is a fixed frequency parameter. We assume the following conditions hold:
\begin{align*}
    n_k &\geq N \\
    \prod_{l=1}^{k-1}Log(n_l) &= o\left(\min_{l \in [0,k]}n_l\right).
\end{align*}
The minimum singular value of the feature matrix $F_k$, denoted 
$\sigma_{\min}(F_{k})$, satisfies the following bound
\begin{align*}
    \sigma_{\min}\left(
F_k
\right)^2
= 
\Theta\left(
	\sqrt{n_0}\beta_kn_k\prod_{l=1}^{k-1}\sqrt{\beta_l}\sqrt{n_l}
	\right)
\end{align*}
w.p. at least
\begin{align*}
    &1 - N(N-1)\exp\left(-
\Omega\left(
\frac{\min_{l \in [0,k]}n_l}{s^kN^2\prod_{l=1}^{k-1}Log(n_l)}
\right)
\right) \\
&\hspace{3cm}- N\sum_{l=1}^k2\exp(-\Omega(n_l))
\end{align*}
\end{Theorem}

The proof of Theorem \ref{sing_val_feat} is provided in Section
\ref{sec_proof_sing_val_feat}.

\section{Experiments}\label{exps}
In this section, we provide experimental evidence to support our theoretical findings. In particular, we measure the scaling of the minimum eigenvalue of the empirical NTK matrix for both cosine and ReLU-activated networks, which are compared to the theoretical predictions in Section \ref{subsec:exps_ntk}. We also evaluate the assumption A4 in Section \ref{subsec:exps_lipschitz}

\subsection{NTK experiments}\label{subsec:exps_ntk}


The prediction made by Theorem \ref{main_result_ntk} suggests that if a cosine activated neural network has one wide layer of width $n_k$, for $1 \leq k \leq L-1$, 
then as we increase $n_k$ while keeping all other widths fixed,
the minimum eigenvalue of the empirical NTK matrix is bounded below by a term of the 
form $\Omega(n_k^{3/2})$. 
This section aims to experimentally verify this theoretical
prediction and contrast it with the prediction made for a deep ReLU network from 
\cite{nguyen2021tight}, which shows that the minimum eigenvalue of a ReLU network would scale linearly.


\paragraph{NTK analysis where $n_1 = 8N$.} In Fig.~\ref{fig:ntk1}, we compare a 3-layer cosine activated neural network with a 3-layer ReLU-activated neural network. We fixed the widths of the input and output layers as $(n_0, n_2)$ and let the width of the middle layer, $n_1$, vary according to the relation $n_1 = 8N$. The cosine activated network used a fixed frequency parameter $s = 5$. Both networks were initialized using He's initialisation, where the weights
$(W_l)_{i,j} \sim \mathcal{N}(0, \frac{2}n_{l-1})$. We then plotted the minimum eigenvalue of the
empirical NTK
$\lambda_{\min}(K_3)$ of both networks, and compared them to curves of the form
$\mathcal{O}((8x)^{1.5})$ and $\mathcal{O}(8x)$.
As predicted by Theorem \ref{main_result_ntk}, we observed that the minimum eigenvalue $\lambda_{\min}(K_3)$ grew faster than $\mathcal{O}((8x)^{1.5})$, and as predicted by 
\cite{nguyen2021tight} the ReLU network grew faster than $\mathcal{O}(8x)$. 


Furthermore, Theorem \ref{main_result_ntk} predicts that the minimum eigenvalue $\lambda_{\min}(K_3)$ for a cosine activated network should grow at a faster rate than a ReLU-activated network. 
Fig.~\ref{fig:ntk1} clearly shows that the minimum eigenvalue of the cosine network grows much faster than the ReLU network.

\begin{figure}[t]
\includegraphics[width=1.0\linewidth]{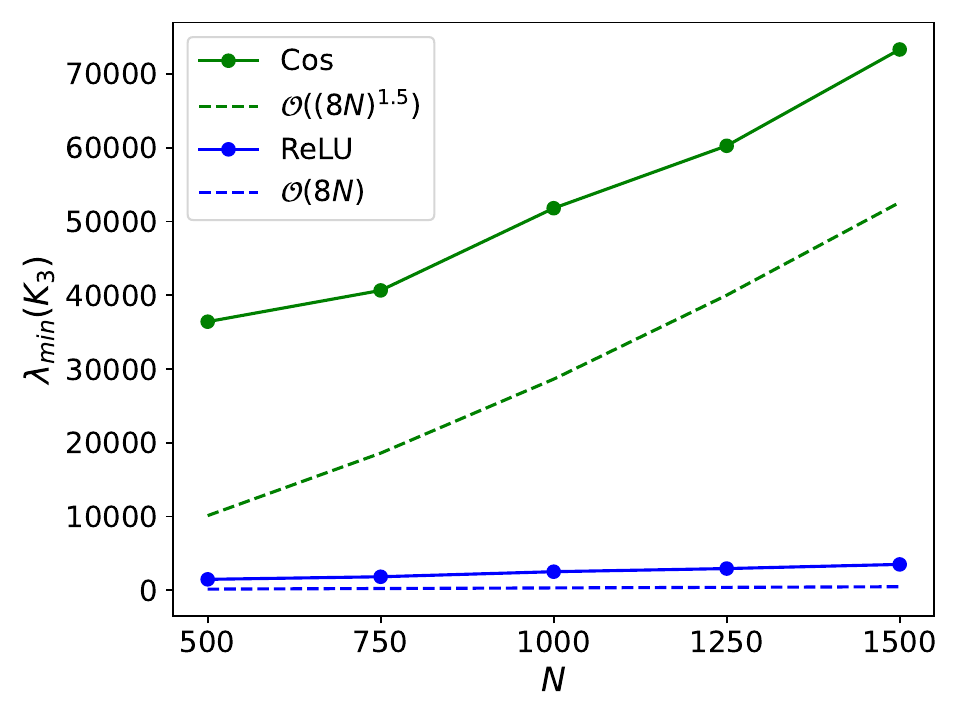}
\hspace{-0.4cm}
\caption{The minimum eigenvalue of empirical NTK $\lambda_{min}(K_3)$ where $n_0 = 400$, $n_{1} = 8N$,  and $n_2=400$. As predicted by Theorem \ref{main_result_ntk}, $\lambda_{min}(K_3)$ for a cosine activated network grows much faster than a ReLU-activated network. }\label{fig:ntk1}
\end{figure}


\paragraph{NTK analysis where $n_1=15N$.} We conducted a follow-up experiment using a larger scaling for the width of $n_1$, where $n_1 = 15N$, while keeping all other parameters constant. Fig.~\ref{fig:ntk2} clearly demonstrates that the prediction of Theorem \ref{main_result_ntk} holds, furthermore the gap between the growth of the cosine and ReLU networks is even more pronounced.

\begin{figure}[t]
\includegraphics[width=1.0\linewidth]{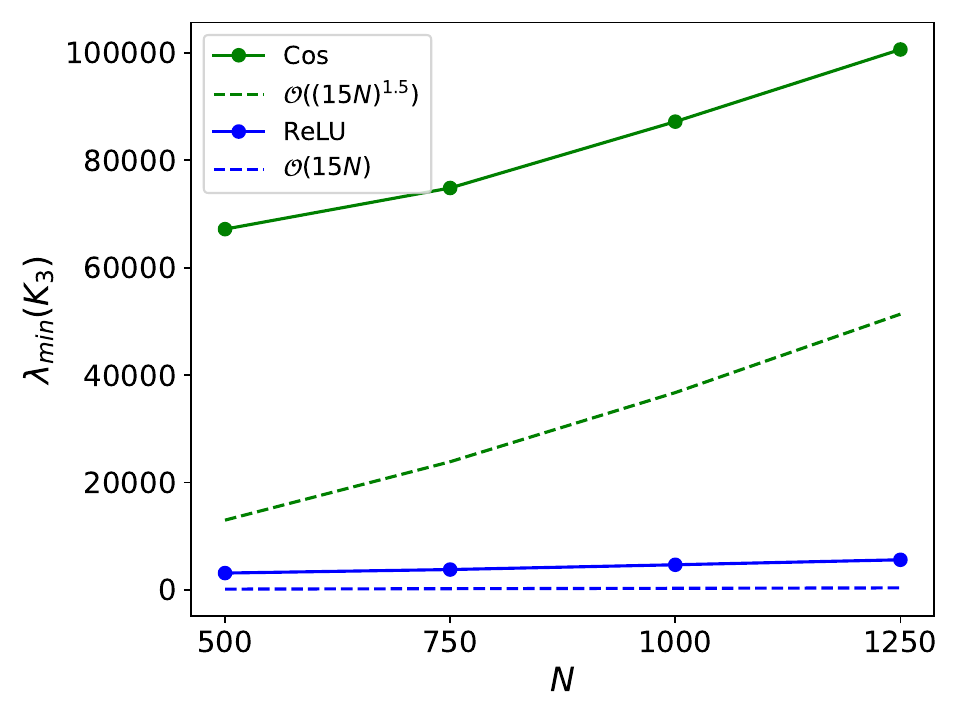}
\hspace{-0.4cm}
\caption{The minimum eigenvalue of empirical NTK $\lambda_{min}(K_3)$ where $n_0=400$, $n_{1} = 15N$, 
 and $n_2 = 400$. As predicted by Theorem \ref{main_result_ntk}, $\lambda_{min}(K_L)$ for a cosine activated network grows much faster than a ReLU-activated network.}\label{fig:ntk2}
\end{figure}

\subsection{Empirical Lipschitz constant}\label{subsec:exps_lipschitz}


In this experiment, we empirically verified the assumption A4. The Lipschitz constant of the $k$-layer function $f_k$
can be expressed as the supremum, over each point in data space, of the operator norm of the Jacobian matrix as
\begin{equation*}
    \vert\vert f_k\vert\vert_{Lip} = \sup_{x \in \R^{n_0}}\vert\vert 
    J(f_k)(x)\vert\vert_{op}.
\end{equation*}
The exact computation of the Lipschitz constant of a deep network is considered as an NP-hard problem~\cite{virmaux2018lipschitz}. Therefore, for our experiment, we will consider the empirical Lipschitz constant. We obtain a sampled data set $X$, sampled from a fixed distribution $\mathcal{P}$; see Section \ref{notations}. 
We then define the empirical Lipschitz constant of $f_k$ as
\begin{equation*}
    \vert\vert f_k\vert\vert_{ELip} = \max_{x \in X}\vert\vert 
    J(f_k)(x)\vert\vert_{op}.
\end{equation*}
Note that the empirical Lipschitz constant of $f_k$ is a lower bound for the true Lipschitz constant of $f_k$.

\paragraph{Empirical Lipschitz constant of a cosine activated network.} We computed the empirical Lipschitz constant of a $4-$layer cosine activated network which has a fixed frequency $s=30$, $f_3$ over $1000$ data points, drawn from a Gaussian distribution
$\mathcal{N}(0, 1)$. 
The widths of the layers were fixed as $n_1 = n_2 = 64$, $n_4 = 1$, and $n_3$ was varied from $64$ to $2048$. We considered two different data dimensions, namely 
$n_0 = 200$ and $n_0 = 400$.
Fig.~\ref{fig:lipschitz_constant_cos} clearly shows that the empirical Lipschitz constant grows much slower with width than a term that grows $\mathcal{O}(n_3^{1/2})$. 
This empirically supports the assumption A4 and demonstrates that the bound given in assumption is an extremely loose bound.

\begin{figure}
\includegraphics[width=1.0\linewidth]{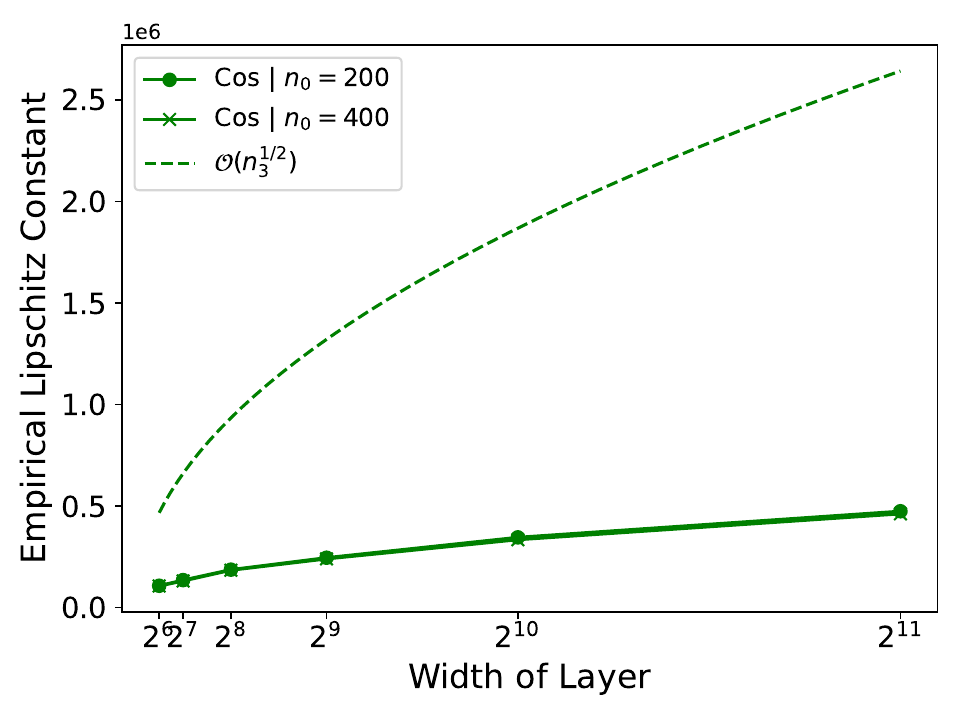}
\hspace{-0.4cm}
\caption{The empirical Lipschitz constant of a cosine activated network over $1000$ data points, where $n_1 = n_2 = 64$, $n_4 = 1$ and $n_3$ varying from 64 to 2048,  when $n_0 = 200$ and $n_0 =400$. This plot empirically confirms the assumption A4.  }\label{fig:lipschitz_constant_cos}
\end{figure}


\paragraph{Comparison of empirical Lipschitz constant of a cosine and a ReLU-activated network.}We compared the empirical Lipschitz constant of a cosine activated neural network with a ReLU-activated network. We note that a ReLU-activated network is only differentiable outside a set of measure zero. Hence, the Jacobian of $f_k$ will represent the true Jacobian only at points where $f_k \neq 0$. 
We used the same parameters as the previous experiment.
Fig.~\ref{fig:lipschitz_constant_all} shows that a cosine activated network has a much larger empirical Lipschitz constant when compared to a ReLU-activated network.

\begin{figure}
\includegraphics[width=1.0\linewidth]{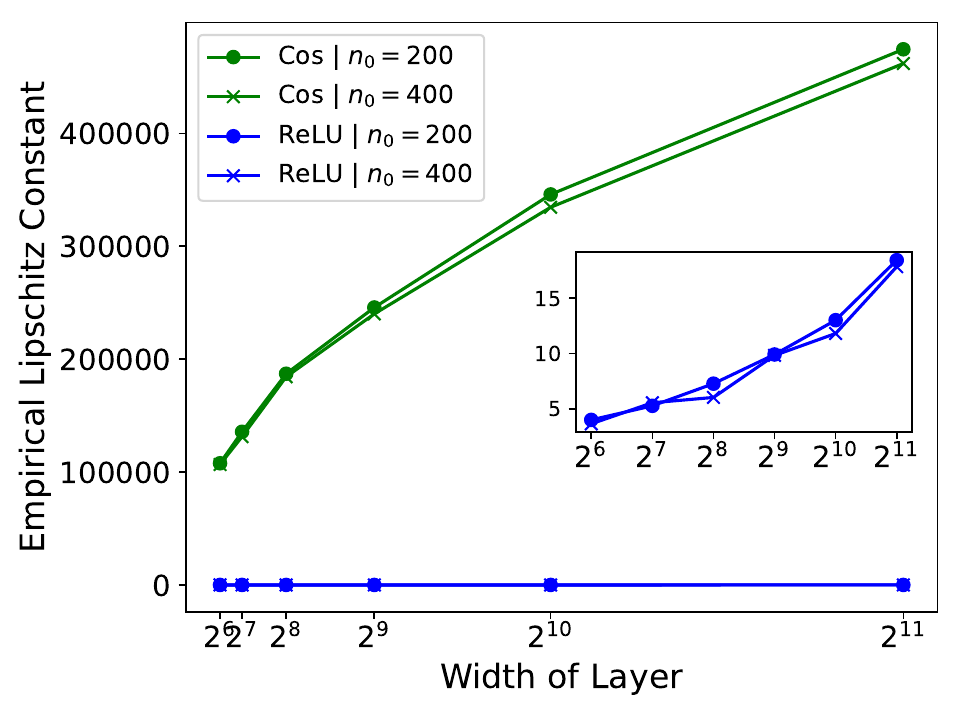}
\hspace{-0.2cm}
\caption{The empirical Lipschitz constant of cosine and ReLU-activated networks over $1000$ data points, where $n_1 = n_2 = 64$, $n_4 = 1$ and $n_3$ varying from 64 to 2048,  when $n_0 = 200$ and $n_0 =400$. \textit{Zoom inset}: The empirical Lipschitz constant of ReLU-activated network.}\label{fig:lipschitz_constant_all}
\end{figure}

\section{Related Work.}\label{related_work}



The application of random matrix theory to deep learning has seen surge of interest in recent years. Research connecting the generalization error to the spectrum of random matrices has been carried out in \cite{hastie2022surprises, gerace2020generalisation, liao2020random, montanari2022interpolation}. The paper by \cite{pennington2018emergence} applied random matrix theory to investigate the spectrum of the input-output Jacobian of a network. Additionally, random matrix techniques have been utilized in various settings to study the spectrum of the conjugate kernel \cite{pennington2017nonlinear, liao2018spectrum}.

While these works have applied techniques from random matrix theory to problems in deep learning, the study of the NTK via random matrix methods is an under-explored area. The work \cite{tancik2020fourier} used the NTK to show why conventional ReLU activated networks admit spectral bias and hence cannot learn high frequency components of a signal.
The work by \cite{montanari2022interpolation} provides a lower bound on the smallest eigenvalue of the NTK matrix but is limited to the restrictive setting of a two layer model. To the best of our knowledge, the only papers that have obtained tight 2-sided bounds on the minimum eigenvalue of the NTK matrix of a ReLU-activated network are \cite{nguyen2021tight, nguyen2020global}. These papers consider a more general setting that only requires one wide layer of at least linear growth in the number of training samples, at any position in the network. In contrast, our work contributes to this growing field by focusing specifically on the theoretical study of periodically activated networks.

\section{Discussion and Conclusion.}

In this paper, we delved into the investigation of the empirical NTK of a periodically activated neural network. We established lower and upper bounds on the minimum eigenvalue of the empirical NTK matrix in the finite width setting. 
These bounds work in the general setting where there is at least one wide layer $n_k$ present in the neural network, positioned anywhere between the input and output layers. In such a setting, the main result of our paper implies that the scaling of the minimum eigenvalue value follows $\Theta(n_k^{3/2})$. This is in contrast to previous results obtained for ReLU networks in \cite{nguyen2021tight}, where the scaling was found to be $\Theta(n_k)$. 
Our results shows that at initialization, the spectrum of the empirical NTK of a periodically activated network has a larger spectral gap than a ReLU-activated network, making it better conditioned for training via gradient decent methods, as described in Equation \ref{ntk_loss}. 
We also verified these theoretical predictions through experiments and contrasted them with results for ReLU-activated networks. Finally, we also provided an application to the memorization capacity of periodically activated networks, which generalizes the memorization result of \cite{nguyen2021tight} to the case of periodically activated networks. Overall, our study provides a deeper understanding of the properties of periodically activated neural networks and their impact on deep learning. We should mention that the original paper \cite{sitzmann2020implicit} employed a sinusoidal activated network for many of their applications. Using the angle formula $sin(x) = cos(x + \frac{\pi}{2})$, we see that a sinusoid is nothing but a phase shifted cosine. Hence the results of this paper, using a cosine activation, go through verbatim for a sine activation. 

A limitation of this work is the assumption of a Lipschitz constant bound on the networks being considered, as outlined in assumption A4. Despite this limitation, 
the assumption allows us to uncover valuable insights into the behavior of  periodically activated neural networks which are crucial in explaining the potential advantages they offer over networks using traditional ReLU activation function. This paves the way for further research and development in this area.

Finally, it is worth noting that there have been several works in recent years that have investigated the use of non-traditional activations in neural network, and have shown that they can exhibit superior performance over ReLU activations. For example, \cite{ramasinghe2022beyond} introduced the use of Gaussian activation functions and showed they are more robust to random initializations than periodic activations, when trained with gradient based methods. Another work \cite{chng2022gaussian} applied Gaussian-activated coordinate networks to the problem of reconstructing neural radiance fields. An interesting line of future research would be to analyze the NTK of such networks and compare them to the case of periodic and ReLU-activated networks. Furthermore, in order to ensure convergence guarantees when training with periodically activated networks, it would be necessary to track the minimum eigenvalue of the NTK during the training process, which we do not address in this paper. This presents an intriguing problem for future research. Overall, the study of non-traditional activation functions and their properties is an active and growing area of research, with the potential to lead to new and improved coordinate networks.

\nocite{langley00}

\bibliography{example_paper}

\begin{thebibliography}{30}
\providecommand{\natexlab}[1]{#1}
\providecommand{\url}[1]{\texttt{#1}}
\expandafter\ifx\csname urlstyle\endcsname\relax
  \providecommand{\doi}[1]{doi: #1}\else
  \providecommand{\doi}{doi: \begingroup \urlstyle{rm}\Url}\fi

\bibitem[Allen-Zhu et~al.(2019)Allen-Zhu, Li, and Song]{allen2019convergence}
Allen-Zhu, Z., Li, Y., and Song, Z.
\newblock A convergence theory for deep learning via over-parameterization.
\newblock In \emph{International Conference on Machine Learning}, pp.\  242--252. PMLR, 2019.

\bibitem[Arora et~al.(2019)Arora, Du, Hu, Li, and Wang]{arora2019fine}
Arora, S., Du, S., Hu, W., Li, Z., and Wang, R.
\newblock Fine-grained analysis of optimization and generalization for overparameterized two-layer neural networks.
\newblock In \emph{International Conference on Machine Learning}, pp.\  322--332. PMLR, 2019.

\bibitem[Chen et~al.(2021)Chen, Liu, and Wang]{chen2021learning}
Chen, Y., Liu, S., and Wang, X.
\newblock Learning continuous image representation with local implicit image function.
\newblock In \emph{Proceedings of the IEEE/CVF conference on computer vision and pattern recognition}, pp.\  8628--8638, 2021.

\bibitem[Chng et~al.(2022)Chng, Ramasinghe, Sherrah, and Lucey]{chng2022gaussian}
Chng, S.-F., Ramasinghe, S., Sherrah, J., and Lucey, S.
\newblock Gaussian activated neural radiance fields for high fidelity reconstruction and pose estimation.
\newblock In \emph{Computer Vision--ECCV 2022: 17th European Conference, Tel Aviv, Israel, October 23--27, 2022, Proceedings, Part XXXIII}, pp.\  264--280. Springer, 2022.

\bibitem[Davidson \& Szarek(2001)Davidson and Szarek]{davidson2001local}
Davidson, K.~R. and Szarek, S.~J.
\newblock Local operator theory, random matrices and banach spaces.
\newblock \emph{Handbook of the geometry of Banach spaces}, 1\penalty0 (317-366):\penalty0 131, 2001.

\bibitem[Du et~al.(2019)Du, Lee, Li, Wang, and Zhai]{du2019gradient}
Du, S., Lee, J., Li, H., Wang, L., and Zhai, X.
\newblock Gradient descent finds global minima of deep neural networks.
\newblock In \emph{International conference on machine learning}, pp.\  1675--1685. PMLR, 2019.

\bibitem[Du et~al.(2018)Du, Zhai, Poczos, and Singh]{du2018gradient}
Du, S.~S., Zhai, X., Poczos, B., and Singh, A.
\newblock Gradient descent provably optimizes over-parameterized neural networks.
\newblock \emph{arXiv preprint arXiv:1810.02054}, 2018.

\bibitem[Gerace et~al.(2020)Gerace, Loureiro, Krzakala, M{\'e}zard, and Zdeborov{\'a}]{gerace2020generalisation}
Gerace, F., Loureiro, B., Krzakala, F., M{\'e}zard, M., and Zdeborov{\'a}, L.
\newblock Generalisation error in learning with random features and the hidden manifold model.
\newblock In \emph{International Conference on Machine Learning}, pp.\  3452--3462. PMLR, 2020.

\bibitem[Hastie et~al.(2022)Hastie, Montanari, Rosset, and Tibshirani]{hastie2022surprises}
Hastie, T., Montanari, A., Rosset, S., and Tibshirani, R.~J.
\newblock Surprises in high-dimensional ridgeless least squares interpolation.
\newblock \emph{The Annals of Statistics}, 50\penalty0 (2):\penalty0 949--986, 2022.

\bibitem[Horn et~al.(1994)Horn, Horn, and Johnson]{horn1994topics}
Horn, R.~A., Horn, R.~A., and Johnson, C.~R.
\newblock \emph{Topics in matrix analysis}.
\newblock Cambridge university press, 1994.

\bibitem[Langley(2000)]{langley00}
Langley, P.
\newblock Crafting papers on machine learning.
\newblock In Langley, P. (ed.), \emph{Proceedings of the 17th International Conference on Machine Learning (ICML 2000)}, pp.\  1207--1216, Stanford, CA, 2000. Morgan Kaufmann.

\bibitem[Li et~al.(2022)Li, Li, Sitzmann, Agrawal, and Torralba]{li20223d}
Li, Y., Li, S., Sitzmann, V., Agrawal, P., and Torralba, A.
\newblock 3d neural scene representations for visuomotor control.
\newblock In \emph{Conference on Robot Learning}, pp.\  112--123. PMLR, 2022.

\bibitem[Liao \& Couillet(2018)Liao and Couillet]{liao2018spectrum}
Liao, Z. and Couillet, R.
\newblock On the spectrum of random features maps of high dimensional data.
\newblock In \emph{International Conference on Machine Learning}, pp.\  3063--3071. PMLR, 2018.

\bibitem[Liao et~al.(2020)Liao, Couillet, and Mahoney]{liao2020random}
Liao, Z., Couillet, R., and Mahoney, M.~W.
\newblock A random matrix analysis of random fourier features: beyond the gaussian kernel, a precise phase transition, and the corresponding double descent.
\newblock \emph{Advances in Neural Information Processing Systems}, 33:\penalty0 13939--13950, 2020.

\bibitem[Mildenhall et~al.(2021)Mildenhall, Srinivasan, Tancik, Barron, Ramamoorthi, and Ng]{mildenhall2021nerf}
Mildenhall, B., Srinivasan, P.~P., Tancik, M., Barron, J.~T., Ramamoorthi, R., and Ng, R.
\newblock Nerf: Representing scenes as neural radiance fields for view synthesis.
\newblock \emph{Communications of the ACM}, 65\penalty0 (1):\penalty0 99--106, 2021.

\bibitem[Montanari \& Zhong(2022)Montanari and Zhong]{montanari2022interpolation}
Montanari, A. and Zhong, Y.
\newblock The interpolation phase transition in neural networks: Memorization and generalization under lazy training.
\newblock \emph{The Annals of Statistics}, 50\penalty0 (5):\penalty0 2816--2847, 2022.

\bibitem[Nguyen et~al.(2021)Nguyen, Mondelli, and Montufar]{nguyen2021tight}
Nguyen, Q., Mondelli, M., and Montufar, G.~F.
\newblock Tight bounds on the smallest eigenvalue of the neural tangent kernel for deep relu networks.
\newblock In \emph{International Conference on Machine Learning}, pp.\  8119--8129. PMLR, 2021.

\bibitem[Nguyen \& Mondelli(2020)Nguyen and Mondelli]{nguyen2020global}
Nguyen, Q.~N. and Mondelli, M.
\newblock Global convergence of deep networks with one wide layer followed by pyramidal topology.
\newblock \emph{Advances in Neural Information Processing Systems}, 33:\penalty0 11961--11972, 2020.

\bibitem[Oymak \& Soltanolkotabi(2020)Oymak and Soltanolkotabi]{oymak2020toward}
Oymak, S. and Soltanolkotabi, M.
\newblock Toward moderate overparameterization: Global convergence guarantees for training shallow neural networks.
\newblock \emph{IEEE Journal on Selected Areas in Information Theory}, 1\penalty0 (1):\penalty0 84--105, 2020.

\bibitem[Pennington \& Worah(2017)Pennington and Worah]{pennington2017nonlinear}
Pennington, J. and Worah, P.
\newblock Nonlinear random matrix theory for deep learning.
\newblock \emph{Advances in neural information processing systems}, 30, 2017.

\bibitem[Pennington et~al.(2018)Pennington, Schoenholz, and Ganguli]{pennington2018emergence}
Pennington, J., Schoenholz, S., and Ganguli, S.
\newblock The emergence of spectral universality in deep networks.
\newblock In \emph{International Conference on Artificial Intelligence and Statistics}, pp.\  1924--1932. PMLR, 2018.

\bibitem[Ramasinghe \& Lucey(2022)Ramasinghe and Lucey]{ramasinghe2022beyond}
Ramasinghe, S. and Lucey, S.
\newblock Beyond periodicity: Towards a unifying framework for activations in coordinate-mlps.
\newblock In \emph{Computer Vision--ECCV 2022: 17th European Conference, Tel Aviv, Israel, October 23--27, 2022, Proceedings, Part XXXIII}, pp.\  142--158. Springer, 2022.

\bibitem[Schur(1911)]{schur1911bemerkungen}
Schur, J.
\newblock Bemerkungen zur theorie der beschr{\"a}nkten bilinearformen mit unendlich vielen ver{\"a}nderlichen.
\newblock 1911.

\bibitem[Sitzmann et~al.(2020)Sitzmann, Martel, Bergman, Lindell, and Wetzstein]{sitzmann2020implicit}
Sitzmann, V., Martel, J., Bergman, A., Lindell, D., and Wetzstein, G.
\newblock Implicit neural representations with periodic activation functions.
\newblock \emph{Advances in Neural Information Processing Systems}, 33:\penalty0 7462--7473, 2020.

\bibitem[Skorokhodov et~al.(2021)Skorokhodov, Ignatyev, and Elhoseiny]{skorokhodov2021adversarial}
Skorokhodov, I., Ignatyev, S., and Elhoseiny, M.
\newblock Adversarial generation of continuous images.
\newblock In \emph{Proceedings of the IEEE/CVF Conference on Computer Vision and Pattern Recognition}, pp.\  10753--10764, 2021.

\bibitem[Song \& Yang(2019)Song and Yang]{song2019quadratic}
Song, Z. and Yang, X.
\newblock Quadratic suffices for over-parametrization via matrix chernoff bound.
\newblock \emph{arXiv preprint arXiv:1906.03593}, 2019.

\bibitem[Tancik et~al.(2020)Tancik, Srinivasan, Mildenhall, Fridovich-Keil, Raghavan, Singhal, Ramamoorthi, Barron, and Ng]{tancik2020fourier}
Tancik, M., Srinivasan, P., Mildenhall, B., Fridovich-Keil, S., Raghavan, N., Singhal, U., Ramamoorthi, R., Barron, J., and Ng, R.
\newblock Fourier features let networks learn high frequency functions in low dimensional domains.
\newblock \emph{Advances in Neural Information Processing Systems}, 33:\penalty0 7537--7547, 2020.

\bibitem[Vershynin(2018)]{vershynin2018high}
Vershynin, R.
\newblock \emph{High-dimensional probability: An introduction with applications in data science}, volume~47.
\newblock Cambridge university press, 2018.

\bibitem[Virmaux \& Scaman(2018)Virmaux and Scaman]{virmaux2018lipschitz}
Virmaux, A. and Scaman, K.
\newblock Lipschitz regularity of deep neural networks: analysis and efficient estimation.
\newblock \emph{Advances in Neural Information Processing Systems}, 31, 2018.

\bibitem[Zou \& Gu(2019)Zou and Gu]{zou2019improved}
Zou, D. and Gu, Q.
\newblock An improved analysis of training over-parameterized deep neural networks.
\newblock \emph{Advances in neural information processing systems}, 32, 2019.

\end{thebibliography}
\bibliographystyle{icml2023}

\newpage
\appendix
\onecolumn
\section{Preliminary lemmas.}\label{app_prelims}

In this section of the appendix we prove several lemmas that are crucial for the probabilistic analysis of cosine activated neural networks. We will fix a depth $L$ cosine activated neural network $f_l$ satisfying assumptions A1-A4, 
see Section \ref{notations}. 
We will assume that all the Gaussian integrals we perform are normalised so that a factor of $\sqrt{\pi}$ does not show. In other words, an integral of the form
$\int_{\R}e^{-x^2}dx$ will be have be normalised to 
$\frac{1}{\sqrt{pi}}\int_{\R}e^{-x^2}dx$. We note that this does not affect any of our results as all the results are asymptotic results and hence the factor of $\sqrt{\pi}$
can be absorbed into a constant. Furthermore, in many of the proofs positive constants
$C > 0$ will arise, such constants may change from line to line. Again this is not a cause for concern as our analysis is in the asymptotic regime.

We will also at times need to make use of the sub-exponential and sub-Gaussian norms, which we now describe.
Given a sub-exponential random variable $X$, define 
\begin{equation*}
    \vert\vert X\vert\vert_{\psi_1} = \inf\{t > 0: 
    \mathbb{E}[exp(\vert X\vert/t)] \leq 2\}.
\end{equation*}
Given a sub-Gaussian random variable $X$ define the sub-Gaussian norm
\begin{equation*}
    \vert\vert X\vert\vert_{\psi_2} = \inf\{t > 0: 
    \mathbb{E}[exp(\vert X^2\vert/t^2)] \leq 2\}.
\end{equation*}

\begin{lemma}\label{lemmac1_cos}
Fix $0 \leq k \leq L-1$ and 
assume $x \sim \mathcal{P}$. Then
\begin{equation*}
	\vert\vert f_k(x)\vert\vert_2^2 = \Theta\bigg{(}s\cdot \sqrt{n_0}
	\prod_{l=1}^{k-1}\sqrt{\beta_l}\sqrt{n_l}\beta_k n_k\bigg{)}
\end{equation*}
w.p. $\geq 1 - \sum_{l=1}^k2exp(-\Omega(sn_l)) - 2exp(-\Omega(\sqrt{n_0}))$ over
$(W_l)_{l=1}^k$ and $x$.

Furthermore 
\begin{equation*}
	\mathbb{E}_{x\sim \mathcal{P}}\vert\vert f_k(x)\vert\vert^2_2 = 
	\Theta\bigg{(}s\cdot \sqrt{n_0}
	\prod_{l=1}^{k-1}\sqrt{\beta_l}\sqrt{n_l}\beta_k n_k\bigg{)}
\end{equation*} 
w.p. $\geq 1 - \sum_{l=1}^k2exp(-\Omega(sn_l))$ over
$(W_l)_{l=1}^k$.
\end{lemma}

\begin{proof}
The proof will be by induction. From the data assumptions, it is clear that the lemma is true for $k = 0$. Assume the lemma holds for $k-1$, we prove it for $k$.
The proof proceeds by conditioning on the event $(W_l)_{l=1}^{k-1}$ and obtaining
bounds over $W_k$. Then by the induction hypothesis and intersecting over the two events the result will follow.

We have that
$W_k \in R^{n_{k-1}\times n_{k}}$, so we can write $W_k = [w_1,\ldots ,w_{n_k}]$, where each $w_i \in \R^{n_{k-1}}$ and $w_i \sim 
\mathcal{N}(0, \beta_k^2I_{n_{k-1}})$. We then estimate
\begin{equation*}
	\vert\vert f_k(x)\vert\vert_2^2 = \sum_{i=1}^k
	\vert\vert f_{k,i}^2\vert\vert^2.
\end{equation*}
Taking the expectation, we have
\begin{align*}
	\mathbb{E}_{W_k}\vert\vert f_{k}\vert\vert_2^2 &= 
	\sum_{i=1}^{n_k}\mathbb{E}_{w_i}[f_{k,i}(x)^2]\text{, by independence} \\
	&= n_k \mathbb{E}_{w_i}[f_{k,i}(x)^2].
\end{align*}
By definition $f_{k,j}(x) = \phi(\langle w_j,f_{k-1}(x)\rangle)$. Note that the random variable $\langle w_j,f_{k-1}(x)\rangle$ is a univariate random variable distributed according to 
$\mathcal{N}(0, \beta_k^2\vert\vert f_{k-1}(x)\vert\vert_2^2)$. Therefore, the above expectation can be estimated by estimating the integral
$\int_{\R}cos^2(w)
e^{\frac{-w^2}{\beta_k^2\vert\vert f_{k-1}(x)\vert\vert_2^2}}dw$. Using the identity $cos^2(sw) = \frac{1}{2} + \frac{cos(2sw)}{2}$, we then have
\begin{equation}
\int_{\R}cos^2(w)e^{\frac{-w^2}{\beta_k^2\vert\vert f_{k-1}(x)\vert\vert_2^2}}dw
=
\int_{\R}\frac{1}{2}e^{\frac{-w^2}{\beta_k^2\vert\vert f_{k-1}(x)\vert\vert_2^2}}dw + 
\int_{\R}\frac{cos(2sw)}{2}
e^{\frac{-w^2}{\beta_k^2\vert\vert f_{k-1}(x)\vert\vert_2^2}}dw.
\end{equation}  
The first integral on the right can be estimated using the usual integral of a Gaussian to give
\begin{equation*}
	\int_{\R}\frac{1}{2}e^{\frac{-w^2}{\beta_k^2\vert\vert f_{k-1}(x)\vert\vert_2^2}}dw = 
	\frac{\beta_k\vert\vert f_{k-1}\vert\vert_2}{2}.
\end{equation*}
In order to compute the second integral we proceed as follows
\begin{align*}
	\frac{1}{2}\int_{\R}cos(2sw)
	e^{\frac{-w^2}{\beta_k^2\vert\vert f_{k-1}(x)\vert\vert_2^2}}dw + 
	\frac{i}{2}\int_{\R}sin(2sw)
	e^{\frac{-w^2}{\beta_k^2\vert\vert f_{k-1}(x)\vert\vert_2^2}}dw &=
	\frac{1}{2}\int_{\R}e^{2isw}
	e^{\frac{-w^2}{\beta_k^2\vert\vert f_{k-1}(x)\vert\vert_2^2}}dw \\
	&=
	\frac{1}{2}\int_{\R}
	e^{-\big{(}\frac{w^2}{\beta_k^2\vert\vert f_{k-1}(x)\vert\vert_2^2} - 
	2isw\big{)}}dw.
\end{align*}
By completing the square, we have
\begin{equation}
	\frac{w^2}{\beta_k^2\vert\vert f_{k-1}\vert\vert^2_2} - 2isw = 
	\bigg{(} 
	\frac{w}{\beta_k\vert\vert f_{k-1}\vert\vert_2} - 
	i\beta_ks\vert\vert f_{k-1}\vert\vert_2	
	\bigg{)}^2 + \beta_k^2\vert\vert f_{k-1}\vert\vert^2_2s^2.
\end{equation}
Using this we have
\begin{align*}
\frac{1}{2}\int_{\R}
	e^{-\big{(}\frac{w^2}{\beta_k^2\vert\vert f_{k-1}(x)\vert\vert_2^2} - 
	2isw\big{)}}dw &= 
	\frac{1}{2}\int_{\R}e^{\bigg{(} 
	\frac{w}{\beta_k\vert\vert f_{k-1}\vert\vert_2} - 
	i\beta_ks\vert\vert f_{k-1}\vert\vert_2	
	\bigg{)}^2}e^{-\beta_k^2\vert\vert f_{k-1}\vert\vert^2_2s^2}dw \\
	&= 
	\frac{e^{-\beta_k^2\vert\vert f_{k-1}\vert\vert^2_2s^2}}{2}
	\int_{\R}
	e^{\bigg{(} 
	\frac{w}{\beta_k\vert\vert f_{k-1}\vert\vert_2} - 
	i\beta_ks\vert\vert f_{k-1}\vert\vert_2	
	\bigg{)}^2}dw.
\end{align*}
Let $x = \frac{w}{\beta_k\vert\vert f_{k-1}\vert\vert_2} - 
i\beta_ks\vert\vert f_{k-1}\vert\vert_2$, so that $dx = 
\frac{dw}{\beta_k\vert\vert f_{k-1}\vert\vert_2}$. Using this substitution, we 
can evaluate the above integral.
\begin{align*}
	\frac{e^{-\beta_k^2\vert\vert f_{k-1}\vert\vert^2_2s^2}}{2}
	\int_{\R}
	e^{\bigg{(} 
	\frac{w}{\beta_k\vert\vert f_{k-1}\vert\vert_2} - 
	i\beta_ks\vert\vert f_{k-1}\vert\vert_2	
	\bigg{)}^2}dw &=
	\frac{e^{-\beta_k^2\vert\vert f_{k-1}\vert\vert^2_2s^2}}{2}\beta_k
	\vert\vert f_{k-1}\vert\vert_2\int_{\R}e^{-x^2}dx \\
	&= 
	\frac{e^{-\beta_k^2\vert\vert f_{k-1}\vert\vert^2_2s^2}}{2}\beta_k
	\vert\vert f_{k-1}\vert\vert_2
\end{align*}
where we remind the reader of our convention of normalising the integral of the Gaussian so the $\sqrt{\pi}$ will not show up. Thus we obtain
\begin{equation*}
	\mathbb{E}_{w_j}[f_{k,j}(x)^2] = \frac{\beta_k}{2}
	\vert\vert f_{k-1}(x)\vert\vert_2 +
	\frac{e^{-\beta_k^2\vert\vert f_{k-1}(x)\vert\vert^2_2s^2}}{2}\beta_k
	\vert\vert f_{k-1}(x)\vert\vert_2.
\end{equation*}
Using the above equality we obtain the expectation bound
\begin{equation}
	\frac{\beta_k\vert\vert f_{k-1}(x)\vert\vert_2}{2} \leq 
	\mathbb{E}_{w_j}[f_{k,j}(x)^2] \leq \beta_k\vert\vert f_{k-1}\vert\vert_2.
\end{equation}
Using the above expectation we would like to apply Bernstein's inequality, see thm. 2.8.1 of \cite{vershynin2018high}, to obtain a bound on $\vert\vert f_k\vert\vert_2^2$. In order to do this we need to compute the sub-Gaussian norm 
$\vert\vert f_{k,j}(x)^2\vert\vert_{\psi_1} = 
\vert\vert f_{k,j}(x)\vert\vert_{\psi_2}^2$. Since $\vert cos(sx)\vert \leq 1$, 
we have
\begin{equation}
\mathbb{E}_{w_i}\bigg{(} 
exp\big{(}
\frac{f_{k,j}(x)^2}{t^2}
\big{)}
\bigg{)} \leq 
\mathbb{E}_{w_j}\bigg{(} 
exp\big{(}\frac{1}{t^2} \big{)}
\bigg{)} = exp\big{(}\frac{1}{t^2} \big{)}\beta_k^{n_k}.
\end{equation}
By taking $t = \frac{1}{Log\bigg{(}\frac{1}{\beta_k^{n_{k-1}}} \bigg{)}}$ we get
that 
\begin{equation*}
	\mathbb{E}_{w_i}\bigg{(} 
exp\big{(}
\frac{f_{k,j}(x)^2}{t^2}
\big{)}
\bigg{)} \leq 
\mathbb{E}_{w_j}\bigg{(} 
exp\big{(}\frac{1}{t^2} \big{)}
\bigg{)} \leq 1,
\end{equation*}
using the fact that $\beta_k \leq 1$. In particular, we get that 
$\vert\vert f_{k,j}(x)\vert\vert_{\psi_2}^2 \leq \mathcal{O}(1)$.

Applying Bernstein's inequality,  see thm. 2.8.1 of \cite{vershynin2018high}, to
$\sum_{i=1}^{n_k}\big{(}f_{k,i}(x)^2 - \mathbb{E}_{w_i}[f_{k,i}(x)^2] \big{)}$
we obtain
\begin{equation}
\vert 
\sum_{i=1}^{n_k}\big{(}f_{k,i}(x)^2 - \mathbb{E}_{w_i}[f_{k,i}(x)^2] \big{)} 
\vert \leq 
\frac{1}{2}\mathbb{E}_{W_k}\vert\vert f_k(x)\vert\vert^2_2
\end{equation}
w.p $\geq 1 - 2exp(-c\frac{\mathbb{E}_{W_k}\vert\vert f_k(x)\vert\vert^2_2}
{2})$. Thus we find that
\begin{equation}
\frac{1}{2}\mathbb{E}_{W_k}\vert\vert f_k(x)\vert\vert^2_2 \leq 
\vert\vert f_k(x)\vert\vert^2_2 \leq \frac{3}{2}\mathbb{E}_{W_k}
\vert\vert f_k(x)\vert\vert_2^2
\end{equation}
w.p. $\geq 1 - 2exp(-2s\Omega(n_k))$. Taking the intersection of the induction over $(W_l)_{l=1}^{k-1}$ and the even over $W_k$ proves the first part of the lemma.

The proof for $\mathbb{E}_x\vert\vert f_k(x)\vert\vert_2^2$ follows a similar argument using Jensen's inequality
$\vert\vert\mathbb{E}_x[f_{k,i}(x)^2\vert\vert_{\psi_1} \leq 
\mathbb{E}_x\vert\vert f_{k,i}(x)^2\vert\vert_{\psi_1} = \mathcal{O}(1)$.
\end{proof}

\begin{lemma}\label{c.2}
Fix $0 \leq k \leq L-1$ and 
assume $x \sim \mathcal{P}$. Then 
\begin{equation*}
	\vert\vert \mathbb{E}_x[f_k(x)]\vert\vert_2^2 = 
	\Theta(n_0\prod_{l=1}^l\beta_l^2n_l)
\end{equation*}
w.p. $\geq 1 - \sum_{l=1}^k2exp(-\Omega(n_l))$ over $(W_l)_{l=1}^k$.
\end{lemma}

\begin{proof}
By Jensen's inequality 
$\vert\vert \mathbb{E}_x[f_k(x)]\vert\vert^2_2 \leq 
\mathbb{E}_x\vert\vert f_k(x)\vert\vert_2^2$. Thus the upper bound follows from
lemma \ref{lemmac1_cos}. 

The proof of the lower bound follows by induction. The $k = 0$ case following from the data assumption. Assume 
\begin{equation*}
	\vert\vert\mathbb{E}_x[f_k(x)]\vert\vert_2^2 = 
	\Omega(sn_0\prod_{l=1}^{k-1}\beta_k)
\end{equation*}
w.p. $\geq 1 - \sum_{l=1}^{k-1}exp(-\Omega(n_l))$ over $(W_l)_{l=1}^{k-1}$. We condition on the intersection of this event and the event of lemma \ref{lemmac1_cos} for $(W_l)_{l=1}^{k-1}$.

Write $W_k = [w_1,\ldots,w_{n_k}]$ with 
$w_j \sim \mathcal{N}(0, \beta_k^2I_{n_k-1})$ for $1\leq j \leq n_k$. Then
\begin{align*}
	\vert\vert \big{(} 
	\mathbb{E}_x[f_{k,i}(x)]	
	\big{)}^2\vert\vert_{\psi_1} &= 
	\vert\vert
	\mathbb{E}_x[f_{k,i}(x)]	
	\vert\vert^2_{\psi_2} \\
	&\leq 
	\mathbb{E}_x\vert\vert f_{k,i}(x)\vert\vert_{\psi_2}^2 \\
	&\leq
	C\sqrt{d}	
\end{align*}
for some $C > 0$.

Moreover, 
\begin{align*}
	\mathbb{E}_{W_k}
	\vert\vert
	\mathbb{E}_x[f_{k,i}(x)]	
	\vert\vert^2_2 &= \sum_{i=1}^{n_k}
	\mathbb{E}_{w_i}(\mathbb{E}_x[f_{k,i}(x)])^2 \\
	&\geq 
	\sum_{i=1}^{n_k}(\mathbb{E}_x\mathbb{E}_{w_i}[f_{k,i}(x)])^2 \\
	&\geq 
	\frac{\beta_k^2n_k}{4}(\mathbb{E}_x\vert\vert f_{k-1}(x)\vert\vert_2)^2 \\
	&=
	\Omega(sn_0\prod_{l=1}^{k}\beta_l^2n_l)
\end{align*}
where the second inequality is computed using the same technique as in 
lemma \ref{lemmac1_cos}.

Applying Bernstein's inequality, see thm. 2.8.1 of \cite{vershynin2018high}, we get
\begin{equation*}
	\vert\vert \mathbb{E}_x[f_k(x)]\vert\vert_2^2 \geq \frac{1}{2}
	\mathbb{E}_{W_k}\vert\vert \mathbb{E}_x[f_k(x)]\vert\vert_2^2 
	= \Omega(sn_0\prod_{l=1}^{n_k}\beta_l^2n_l)
\end{equation*}
w.p. $\geq 1 - 2exp(-\Omega(n_k))$ over $(W_k)$. Taking the intersection of all the events then finishes the proof.
\end{proof}

\begin{lemma}\label{c3}
Fix $0\leq k \leq L-1$ and assume  $\prod_{l=1}^{k-1}Log(n_l) = o\left(\min_{l \in [0,k]}n_l\right)$.
Then for any $i \in [N]$, we have
\begin{equation*}
\vert\vert f_k(x_i) - \mathbb{E}_x[f_k(x)]\vert\vert_2^2 = 
\Theta\bigg{(} 
\sqrt{n_0}\beta_kn_k\prod_{l=1}^{k-1}\sqrt{\beta_l}\sqrt{n_l}
\bigg{)}
\end{equation*}
w.p. $\geq 
1 - Nexp\bigg{(}-\Omega \bigg{(} 
\frac{\min_{l\in [0,k]}n_l}{\prod_{l=1}^{k-1}log(n_l)}
\bigg{)}
\bigg{)}
- \sum_{l=1}^kexp(-\Omega(n_l))$.
\end{lemma}

\begin{proof}
Let $X : \R^{n_0} \rightarrow \R$ denote the random variable defined 
by $X(x_i) = \vert\vert f_k(x_i) - \mathbb{E}_x[f_k(x)]\vert\vert_2$. 
By assumption A4, we have
\begin{equation*}
	\vert\vert X\vert\vert_{Lip}^2 = \mathcal{O}\bigg{(} 
	\frac{s^k\beta_k n_k\prod_{l=1}^{k-1}\sqrt{\beta_l}
 \sqrt{n_l}\prod_{l=1}^{k-1}Log(n_l)}
	{\min_{l \in [0,k]}n_l}
	\bigg{)}
\end{equation*}
w.p. $\geq 1 - \sum_{l=1}^kexp(-\Omega(n_l))$. 

We use the notation $\mathbb{E}[X] = \mathbb{E}_{x_i}[X(x_i)] = 
\int_{\R^{n_0}}X(x_i)d\mathcal{P}(x_i)$. We then have
\begin{align*}
	\mathbb{E}[X]^2 &= \mathbb{E}[X^2] - 
	\mathbb{E}[\vert X - \mathbb{E}X\vert^2]  \\
	&\geq 
	\mathbb{E}[X^2] - 
	\int_{0}^{\infty}\mathbb{P}(|X - \mathbb{E}X|>\sqrt{t})dt \\
	&\geq 
	\mathbb{E}[X^2] - \int_{0}^{\infty} 
	2exp\bigg{(}\frac{-ct}{\vert\vert X\vert\vert_{Lip}^2} \bigg{)}dt \\
	&=  \mathbb{E}[X^2] - \frac{2}{c}\vert\vert X\vert\vert_{Lip}^2.
\end{align*}

By lemma \ref{c4}, we have w.p. $\geq 1 - \sum_{l=1}^kexp(-\Omega(n_l))$ 
over $(W_l)_{l=1}^k$ that
\begin{equation*}
	\mathbb{E}[X^2] = 
	\Theta\bigg{(}
	\sqrt{n_0}\beta_kn_k\prod_{l=1}^{k-1}\sqrt{\beta_l}\sqrt{n_l}\bigg{)}
\end{equation*}
which implies
\begin{equation*}
\mathbb{E}[X] = 
\Omega\bigg{(}
	\sqrt{\sqrt{n_0}\beta_kn_k\prod_{l=1}^{k-1}\sqrt{\beta_l}\sqrt{n_l}}
	\bigg{)}.
\end{equation*}
Moreover, by Jensen's inequality $\mathbb{E}[X] \leq 
\sqrt{\mathbb{E}[X^2]} = \mathcal{O}\big{(}
\sqrt{\sqrt{n_0}\beta_kn_k\prod_{l=1}^{k-1}\sqrt{\beta_l}\sqrt{n_l}}
\big{)}$. 

Putting the above two asymptotic bounds together we obtain
\begin{equation*}
	\mathbb{E}[X] = \Theta\bigg{(}
	\sqrt{\sqrt{n_0}\beta_kn_k\prod_{l=1}^{k-1}\sqrt{\beta_l}\sqrt{n_l}}
	\bigg{)}
\end{equation*}
w.p. $\geq 1 - \sum_{l=1}^kexp(-\Omega(n_l))$ over $(W_l)_{l=1}^k$.

We condition on the above event and obtain bounds over each sample.
Using Lipschitz concentration, see assumption A3, we have that
$\frac{1}{2}\mathbb{E}[X] \leq X \leq \frac{3}{2}\mathbb{E}[X]$. Therefore, 
\begin{equation*}
X = \Theta\bigg{(}
\sqrt{\sqrt{n_0}\beta_kn_k\prod_{l=1}^{k-1}\sqrt{\beta_l}\sqrt{n_l}}
\bigg{)}
\end{equation*}
w.p. $\geq 1 - exp\bigg{(}-\Omega \bigg{(} 
\frac{\min_{l\in [0,k]}n_l}{\prod_{l=1}^{k-1}log(n_l)}
\bigg{)}
\bigg{)}$. Taking the union bounds over the $N$ samples and intersecting them with the above event over $(W_l)_{l=1}^k$ gives the lemma.
\end{proof}

\begin{lemma}\label{c4}
Fix $0\leq k \leq L-1$. Then 
\begin{equation*}
	\mathbb{E}_x\vert\vert f_k(x) - \mathbb{E}_x[f_k(x)]\vert\vert_2^2 = 
	\Theta\bigg{(}
	\sqrt{n_0}\beta_kn_k\prod_{l=1}^{k-1}\beta_ln_l
	\bigg{)}
\end{equation*}
w.p. $\geq 1 - \sum_{l=1}^{k}exp(-\Omega(n_l))$ over $(W_l)_{l=1}^k$.
\end{lemma}

\begin{proof}
The proof is by induction. Note that the $k = 0$ case is given by the concentration inequality assumption of the data. 

Assume the lemma is true for $k-1$. We condition on this event over 
$(W_l)_{l=1}^{k-1}$ and obtain bounds over $W_k$. Then taking the intersection of the two events we will give a proof of the lemma.

We recall that we write $W_k = [w_1,\ldots, w_{n_k}]$ where 
$w_i \sim \mathcal{N}(0, \beta_k^2I_{n_{k-1}})$. By expanding the squared norm we have
\begin{equation*}
\mathbb{E}_x\vert\vert f_k(x) - \mathbb{E}_x[f_k(x)]\vert\vert_2^2 = 
\sum_{j=1}^{n_k}\mathbb{E}_x(f_{k,j}(x) - \mathbb{E}_x[f_{k,j}(x)])^2.
\end{equation*}
We now take the expectation over $W_k$ to obtain
\begin{equation*}
	\mathbb{E}_{W_k}\mathbb{E}_x\vert\vert f_k(x) - 
	\mathbb{E}_x[f_k(x)]\vert\vert^2_2 =
	\mathbb{E}_{W_k}\mathbb{E}_x\vert\vert f_k(x)\vert\vert_2^2 - 
	\mathbb{E}_{W_k}\vert\vert \mathbb{E}_xf_k(x)\vert\vert_2^2.
\end{equation*}
From the proof of lemma \ref{lemmac1_cos}, we know that
\begin{equation*}
\mathbb{E}_{W_k}\vert\vert f_k(x)\vert\vert_2^2 = 
\frac{\beta_kn_k}{2}\vert\vert f_{k-1}(x)\vert\vert_2.
\end{equation*}
Therefore, we can estimate
\begin{align*}
&\mathbb{E}_{W_k}\mathbb{E}_x\vert\vert f_k(x)\vert\vert_2^2 - 
	\mathbb{E}_{W_k}\vert\vert \mathbb{E}_xf_k(x)\vert\vert_2^2 \\
	&\geq
	\frac{\beta_kn_k}{2}\mathbb{E}_x\vert\vert f_{k-1}(x)\vert\vert_2 - 
	\mathbb{E}_x\mathbb{E}_y
	\sum_{i=1}^{n_k}\mathbb{E}_{w_i}
	\phi(\langle w_i, f_{k-1}(x)\rangle)\phi(\langle w_i, f_{k-1}(y)\rangle) \\
	&=
	\frac{\beta_kn_k}{2}\mathbb{E}_x\vert\vert f_{k-1}(x)\vert\vert_2 - 
	n_k\mathbb{E}_x\mathbb{E}_y
	\mathbb{E}_{w_1}
	\phi(\langle w_1, f_{k-1}(x)\rangle)\phi(\langle w_1, f_{k-1}(y)\rangle) \\
	&\geq
	C\sqrt{n_0}\beta_kn_k\prod_{l=1}^{k-1}\sqrt{\beta_l}\sqrt{n_l} 
	- n_k\beta_k \\
	&=
	C\sqrt{n_0}\beta_kn_k\prod_{l=1}^{k-1}\sqrt{\beta_l}\sqrt{n_l}
\end{align*}
where to get the second inequality we have used lemma \ref{c.2}, Jensen's inequality and the fact that $\vert\phi(x)\vert \leq 1$.
In order to get an upper bound we observe
\begin{align*}
	\mathbb{E}_{W_k}\mathbb{E}_x\vert\vert f_k(x) - 
	\mathbb{E}_x[f_k(x)]\vert\vert^2_2 &\leq 
	\mathbb{E}_{W_k}\mathbb{E}_x\vert\vert f_k(x)\vert\vert_2^2 \\
	&\leq 
	\frac{C\beta_kn_k}{2}\mathbb{E}_x\vert\vert f_{k-1}(x)\vert\vert_2 \\
	&\leq C\sqrt{n_0}\beta_kn_k\prod_{l=1}^{k-1}\sqrt{\beta_l}\sqrt{n_l}.
\end{align*}
Applying Bernstein's inequality, see thm. 2.8.1 of \cite{vershynin2018high}, we get
\begin{equation*}
\frac{1}{2}\mathbb{E}_{W_k}\mathbb{E}_x\vert\vert f_k(x) - 
	\mathbb{E}_x[f_k(x)]\vert\vert^2_2 \leq 
	\mathbb{E}_x\vert\vert f_k(x) - 
	\mathbb{E}_x[f_k(x)]\vert\vert^2_2 \leq 
	\frac{3}{2}\mathbb{E}_{W_k}\mathbb{E}_x\vert\vert f_k(x) - 
	\mathbb{E}_x[f_k(x)]\vert\vert^2_2
\end{equation*}
w.p. $\geq 1 - exp(-\Omega(n_k))$ over $W_k$. Taking the intersection of that event, together with the conditioned event over $(W_l)_{l=1}^{k-1}$ gives the statement of the lemma.
\end{proof}

\begin{lemma}\label{c5}
For $0\leq k \leq L-1$ and $x \sim \mathcal{P}$. 
We have that 
\begin{equation*}
	\vert\vert\Sigma_k(x)\vert\vert_F^2 = \Theta\bigg{(} 
	(1- e^{-\beta_k^2s^2})\sqrt{n_0}\beta_kn_k\prod_{l=1}^{k-1}\sqrt{\beta_l}
	\sqrt{n_l}
	\bigg{)}
\end{equation*}
w.p. $\geq 1 - \sum_{l=1}^k2exp(-\Omega(n_l)) - 2exp(-\Omega(\sqrt{n_0}))$.
\end{lemma}

\begin{proof}
We first observe that lemma \ref{lemmac1_cos} implies that 
$\vert\vert f_{k-1}(x)\vert\vert \neq 0$ w.p. $\geq 1 - \sum_{l=0}^{k-1}2exp(-2s\Omega(n_k))$
which in turn implies that 
$f_{k-}(x) \neq 0$ w.h.p. $\geq 1 - \sum_{l=0}^{k-1}2exp(-2s\Omega(n_k))$ over
$(W_l)_{l=1}^{k-1}$ and $x$. We condition on that event and obtain bounds on 
$W_k$. Taking the intersection of the two events will then complete the proof.

Write $W_k = [w_1,\ldots ,w_{n_{k}}]$. Then 
$\vert\vert\Sigma_k(x)\vert\vert_F^2 =
\sum_{i=1}^{n_k}\phi'(\langle f_{k-1}(x), w_i\rangle)^2$. Thus 
$\mathbb{E}_{W_k}\vert\vert\Sigma_k(x)\vert\vert_F^2 = 
n_k\mathbb{E}_{w_1}[\phi'(\langle f_{k-1}(x), w_1\rangle)^2]$, by independence.
Note that $\langle f_{k-1}(x), w_1\rangle$ is a univariate random variable with
distribution $\mathcal{N}(0, \beta_k^2\vert\vert f_{k-1}(x)\vert\vert_2^2)$ and that 
\begin{equation*}
n_k\mathbb{E}_{w_1}[\phi'(\langle f_{k-1}(x), w_1\rangle)^2] = 
sn_k\mathbb{E}_{w_1}[sin(s\langle f_{k-1}(x), w_1\rangle)^2].
\end{equation*}
In order to calculate the above expectation, we need to calculate the integral
\begin{equation*}
\int_{\R}sin^2(sw)e^{-\frac{w^2}{\beta_k^2\vert\vert f_{k-1}\vert\vert_2^2}}dw.
\end{equation*}
This is done by using the identity $sin^2(sw) = 
\frac{1}{2} - \frac{cos(2sw)}{2}$. The integral then become
\begin{align*}
\int_{\R}sin^2(sw)e^{-\frac{w^2}{\beta_k^2\vert\vert f_{k-1}\vert\vert_2^2}}dw 
&= \int_{\R}\bigg{(}\frac{1}{2} - \frac{cos^2(2sw)}{2} \bigg{)}
e^{-\frac{w^2}{\beta_k^2\vert\vert f_{k-1}\vert\vert_2^2}}dw \\
&=
\frac{1}{2}\int_{\R}e^{-\frac{w^2}{\beta_k^2\vert\vert f_{k-1}\vert\vert_2^2}}dw -
\frac{1}{2}\int_{\R}\frac{cos^2(2sw)}{2}
e^{-\frac{w^2}{\beta_k^2\vert\vert f_{k-1}\vert\vert_2^2}}dw.
\end{align*}
The first integral is a standard Gaussian integral
\begin{equation*}
\frac{1}{2}\int_{\R}e^{-\frac{w^2}{\beta_k^2\vert\vert f_{k-1}\vert\vert_2^2}}dw
=
\frac{\beta_k\vert\vert f_{k-1}(x)\vert\vert_2}{2}.
\end{equation*}
The second integral was computed in the proof of lemma \ref{lemmac1_cos}
\begin{equation*}
\frac{1}{2}\int_{\R}\frac{cos^2(2sw)}{2}
e^{-\frac{w^2}{\beta_k^2\vert\vert f_{k-1}\vert\vert_2^2}}dw = 
\bigg{(} 
\frac{e^{-\beta_k^2\vert\vert f_{k-1}\vert\vert^2_2s^2}}{2}
\bigg{)}\beta_k\vert\vert f_{k-1}(x)\vert\vert_2.
\end{equation*}
These computations imply
\begin{equation*}
\mathbb{E}_{W_k}\vert\vert\Sigma_k(x)\vert\vert_F^2 = 
\frac{n_k\beta_k\vert\vert f_{k-1}\vert\vert_2}{2}
\bigg{(} 
1- e^{-\beta_k^2\vert\vert f_{k-1}\vert\vert^2_2s^2}
\bigg{)}.
\end{equation*}
We now apply Hoeffding's inequality, see thm. 2.2.6 of \cite{vershynin2018high}, to get
\begin{equation*}
\bigg{\vert}
\vert\vert\Sigma_k(x)\vert\vert_F^2 - \mathbb{E}_{W_k}\vert\vert \Sigma_k(x)
\vert\vert_F^2
\bigg{\vert} \leq 
\frac{1}{2}\mathbb{E}_{W_k}\vert\vert\Sigma_k(x)\vert\vert_F^2
\end{equation*}
w.p. $\geq 1 - 2exp(-
\frac{\big{(}\mathbb{E}_{W_k}\vert\vert\Sigma_k(x)
\vert\vert_F^2\big{)}^2}{4n_k})$. Using the estimate for 
$\mathbb{E}_{W_k}\vert\vert\Sigma_k(x)\vert\vert_F^2$ that we obtained 
and taking the intersection of the two events
proves 
the lemma.
\end{proof}

\begin{lemma}\label{c6}
For any $1 \leq k \leq L-1$, $k \leq p \leq L-1$ and $x \sim \mathcal{P}$.
We have that
\begin{equation*}
\bigg{\vert}\bigg{\vert}
\Sigma_k(x)\prod_{l=k+1}^pW_l\Sigma_l(x)
\bigg{\vert}\bigg{\vert}^2_F = 
\Theta\bigg{(}
s^2(1 - e^{-\beta_k^2s^2})\sqrt{n_0}\beta_kn_k\beta_pn_p
\prod_{l=1, l\neq k}^{p-1}
\sqrt{\beta_l}\sqrt{n_l}
\bigg{)}
\end{equation*}
w.p. $\geq 1 - \sum_{l=0}^p2exp(-\Omega(n_l))$ over $(W_l)_{l=1}^p$ and 
$x \sim \mathcal{P}$.
\end{lemma}

\begin{proof}
We want to bound 
$\vert\vert \Sigma_k(X) \prod_{l=k+1}^pW_l\Sigma_l(x)\vert\vert^2_F$ for 
$k \leq p \leq l-1$, and any $k \in [L-1]$, $x \sim \mathcal{P}$.

When $p = k$, the quantity reduces to $\vert\vert \Sigma_k(x)\vert\vert_F^2$, which we know how to bound by lemma \ref{c5}. 

Let $B(p) = \Sigma_{k}(x)\prod_{l=k+1}^pW_l\Sigma_l(x) =
\Sigma_{k}(x)\bigg{(}\prod_{l=k+1}^{p-1}W_l\Sigma_l(x)\bigg{)}W_p\Sigma_p(x) =
B(p-1)W_p\Sigma_p(x)$.

Write $W_p = [w_1,\ldots,w_{n_p}]$ and observe that
\begin{equation*}
\vert\vert B(p)\vert\vert_F^2 = \sum_{i=1}^{n_p}
\vert\vert B(p-1)w_i\vert\vert_2^2\phi'(\langle f_{p-1}(x), w_i\rangle)^2.
\end{equation*}
Taking the expectation we obtain
\begin{equation*}
	\mathbb{E}_{W_p}\vert\vert B(p)\vert\vert_F^2 = 
	n_p\mathbb{E}_{w_1}\vert\vert B(p-1)w_1\vert\vert^2_2
	\phi'(\langle f_{p-1}(x), w_i\rangle)^2.
\end{equation*}
The derivative $\phi'(\langle f_{p-1}(x), w_i\rangle)^2 
= s^2sin(\langle f_{p-1}(x), w_i\rangle)^2$. Pick a piecewise non-negative, non-zero, measurable constant function $\chi$ so that 
$0 \leq \chi(x) \leq sin(x)^2$, see figure \ref{fig:low_bound_sin} for a depiction of $\chi$.

\begin{figure}[t]
\includegraphics[width=1.0\linewidth]{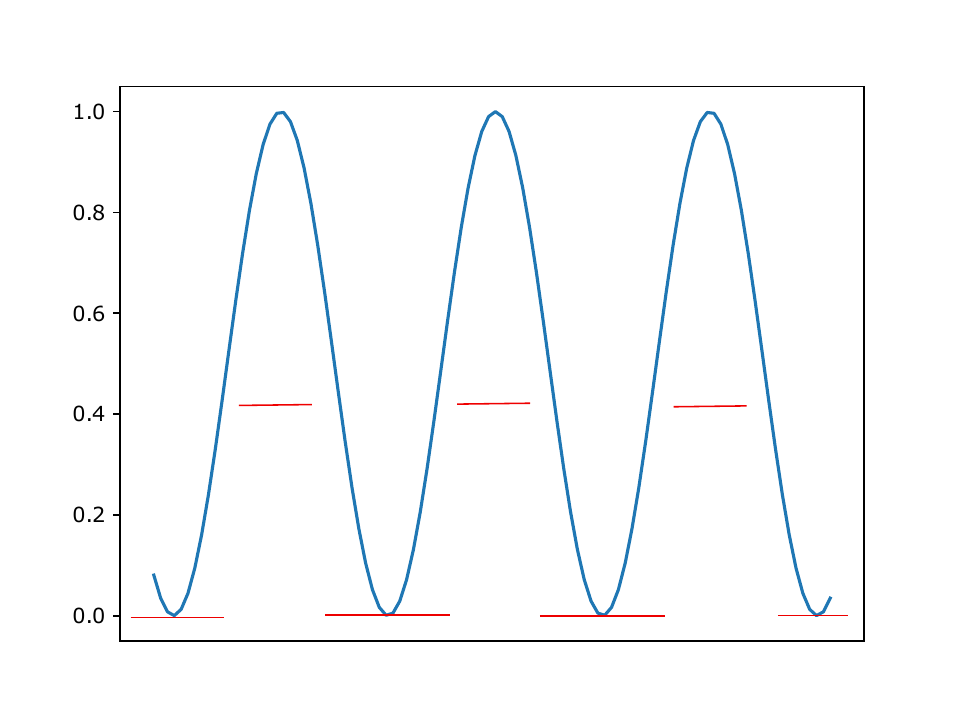}
\vspace{-1.5cm}
\caption{An example of the function $\chi(x)$ in red used in the proof of lemma \ref{c6}. The blue curve represents $sin(x)^2$}\label{chi_bound_sin}
\end{figure}\label{fig:low_bound_sin}

Then observe that
\begin{align*}
\mathbb{E}_{W_p}\vert\vert B(p)\vert\vert_F^2 &\geq 
n_p\mathbb{E}_{w_1}\vert\vert B(p-1)w_1\vert\vert_2^2\chi s^2 \\
&= n_ps^2\beta_p^2\vert\vert B(p-1)\vert\vert_F^2.
\end{align*}
To get an upper bound, we simply observe that $\phi'$ is a bounded function. Therefore, 
\begin{equation*}
\mathbb{E}_{W_p}\vert\vert B(p)\vert\vert_F^2 \leq 
s^2n_p\beta_p\vert\vert B(p-1)\vert\vert_F^2. 
\end{equation*}
By induction we then get 
\begin{equation*}
\mathbb{E}_{W_p}\vert\vert B(p)\vert\vert_F^2 = \Theta\bigg{(}
(1-e^{-\beta_k^2s^2})s^2\sqrt{n_0}\beta_kn_k\beta_pn_p\prod_{l=1, l\neq k}^{p-1}\sqrt{\beta_l}\sqrt{n_l}
\bigg{)}
\end{equation*}
where we have used lemma \ref{c5} in order to induct.

Once we have an expectation bound we can apply Bernstein's inequality, see thm. 2.8.1 of \cite{vershynin2018high}. In order to do this, we need to compute the sub-Gaussian norm. By using the fact that $\phi'(x)^2$ is a bounded function we have
\begin{align*}
	\bigg{\vert}\bigg{\vert}
	\vert\vert B(p-1)w_1\vert\vert^2_2
	\phi'(\langle f_{p-1}(x), w_i\rangle)^2
	\bigg{\vert}\bigg{\vert}_{\psi_1} &\leq 
	C\bigg{\vert}\bigg{\vert} \vert\vert B(p-1)w_1\vert\vert_2
	\bigg{\vert}\bigg{\vert}^2_{\psi_2} \\
	&\leq C\beta_p^2\vert\vert B(p-1)\vert\vert_F^2
\end{align*}
for some $C > 0$.

Once we have the sub-Gaussian norm estimate, we can apply Bernstein's inequality 
to get
\begin{equation*}
\frac{1}{2}\mathbb{E}_{W_p}\vert\vert B(p)\vert\vert_F^2 \leq 
\vert\vert B(p)\vert\vert_F^2 \leq \frac{3}{2}\mathbb{E}_{W_p}
\vert\vert B(p)\vert\vert_F^2
\end{equation*}
w.p. $\geq 1 - 2exp(-\Omega(n_p))$ over $W_p$. Taking the intersection of this event with the previous events over $(W_l)_{l=1}^{p-1}$ and $x$ gives the result.
\end{proof}

\section{Proof of lemma \ref{G-bound}.}\label{sec_G-bound_proof}

The goal of this section is to prove lemma \ref{G-bound}. We start with preliminary lemmas.

\begin{lemma}\label{d2a}
For neural network with activation $\phi(x) = cos(sx)$, we have
$\vert\vert\Sigma_k(x)\vert\vert_{op} \leq s$. 
\end{lemma}

\begin{proof}
$\Sigma_k(x)$ is a diagonal matrix consisting of derivatives of the activation function evaluated at the pre-activated neuron. When the activation is 
$cos(sx)$, the derivative is $-s\cdot sin(sx)$. In particular the derivative is bounded above by $s$. 
\end{proof}

\begin{lemma}\label{d2b}
Let $1 \leq k \leq L-1$ and
let $A = \Sigma_k(x)\prod_{l=k+1}^{L-1}W_l\Sigma_l(x)$. Then we have
\begin{equation*}
\vert\vert A\vert\vert_{op}^2 = \mathcal{O}\bigg{(}
\frac{sn_k}{\min_{l \in [k,L-1]}n_l}\prod_{l=k+1}^{L-1}n_l\beta_l^2
\bigg{)}
\end{equation*}
w.p. $\geq 1 - \sum_{l=0}^k2exp(-\Omega(n_l))$ over $(W_l)_{l=1}^k$ and $x$.
\end{lemma}

\begin{proof}
The proof of this proceeds by induction and follows the exact same strategy 
as in the proof in D.2 in \cite{nguyen2021tight}.

We first note that we have the estimate
\begin{equation}
\vert\vert A\vert\vert_{op} \leq \vert\vert\Sigma(x)\vert\vert_{op}
\bigg{\vert}\bigg{\vert} 
\prod_{l=k+1}^{L-1}W_l\Sigma_l(x)\bigg{\vert}\bigg{\vert}_{op}. 
\end{equation}
We then observe that $\vert\vert \Sigma_k(x)\vert\vert_{op}$ can be bounded
by lemma \ref{d2a}. This means we need only bound 
$\bigg{\vert}\bigg{\vert} 
\prod_{l=k+1}^{L-1}W_l\Sigma_l(x)\bigg{\vert}\bigg{\vert}_{op}$. The proof of this follows by induction on the length $(L-1) - (k+1) = L-k - 2$.

The base case follows by applying operator norm bounds of Gaussian matrices, see Theorem 2.13 in \cite{davidson2001local}.
\begin{equation*}
\big{\vert}\big{\vert} 
W_{L-1}\Sigma_{L-1}(x)\big{\vert}\big{\vert}_{op} \leq 
C(s)\big{\vert}\big{\vert}W_{L-1}\big{\vert}\big{\vert}^2_{op} = 
\mathcal{O}(\beta_{L-1}^2\max\{n_{L-1}, n_{L-2}\}).
\end{equation*}
The general case now follows the $\epsilon$-net argument used in \cite{nguyen2021tight}.
\end{proof}

We are now in a position to give the proof of lemma \ref{G-bound}.

\begin{proof}[\textbf{proof of lemma \ref{G-bound}}]

We need to estimate the quantity 
\begin{equation*}
\bigg{\vert}\bigg{\vert}
\Sigma_k(x)\bigg{(}\prod_{l=k+1}^{L-1}W_l\Sigma_l(x)\bigg{)}W_L
\bigg{\vert}\bigg{\vert}^2_2.
\end{equation*}
Let $A = \Sigma_k(x)\prod_{l=k+1}^{L-1}W_l\Sigma_l(x)$. By lemma \ref{c6}
we have that
\begin{equation*}
\vert\vert A\vert\vert_F^2 = 
\Theta\bigg{(}
s^2(1 - e^{-\beta_k^2s^2})\sqrt{n_0}\beta_kn_k\beta_{L-1}n_{L-1}
\prod_{l=1, l\neq k}^{L-2}
\sqrt{\beta_l}\sqrt{n_l}
\bigg{)}
\end{equation*}
w.p. $\geq 1 - \sum_{l=0}^k2exp(-\Omega(n_l))$ over $(W_{l})_{l=1}^{L-1}$ and 
$x$.

Furthermore by lemma \ref{d2b}, we have the operator norm estimate
\begin{equation*}
\vert\vert A\vert\vert_{op}^2 = \mathcal{O}\bigg{(}
\frac{sn_k}{\min_{l \in [k,L-1]}n_l}\prod_{l=k+1}^{L-1}n_l\beta_l^2
\bigg{)}
\end{equation*}
w.p. $\geq 1 - \sum_{l=0}^k2exp(-\Omega(n_l))$ over $(W_{l})_{l=1}^{L-1}$ and 
$x$.

As $A$ only depends on $(W_{l})_{l=1}^{L-1}$ and $x$, we condition on the above two events over $(W_{l})_{l=1}^{L-1}$ and $x$, and obtain a bound over $W_L$. 
Applying the Hanson-Wright 
inequality, see thm. 6.2.1 of \cite{vershynin2018high}, we get
\begin{equation*}
\bigg{\vert}
\vert\vert AW_L\vert\vert_2^2 - \mathbb{E}_{W_L}\vert\vert AW_L\vert\vert^2_2
\bigg{\vert}
\leq 
\frac{3}{2}
\mathbb{E}_{W_L}\vert\vert AW_L\vert\vert^2_2
\end{equation*}
w.p. $\geq 1 - exp\bigg{(}-\Omega\bigg{(}\frac{\vert\vert A\vert\vert_F^2}
{\max_{i}\vert\vert (AW_L)_i\vert\vert_{\psi_2}}\bigg{)}\bigg{)}$.

Note that $\vert\vert (AW_L)_i\vert\vert_{2}^2 \leq 
\vert\vert B\vert\vert_{op}^2\vert\vert W_L\vert\vert_2^2$. It follows that
for each $i$ that 
$\vert\vert (AW_L)_i\vert\vert_{\psi_2} = 
\mathcal{O}(\vert\vert B\vert\vert_{op})$. We therefore find that
\begin{equation*}
\vert\vert AW_L\vert\vert_2^2 = 
\Theta\bigg{(}
s^2(1 - e^{-\beta_k^2s^2})\sqrt{n_0}\beta_kn_k\beta_{L}n_L
\prod_{l=1, l\neq k}^{L-1}
\sqrt{\beta_l}\sqrt{n_l}
\bigg{)}
\end{equation*}
w.p. $\geq 1 - 2exp(-\Omega(n_k))$. By taking the intersection of this event with the one we conditioned over, we get the result.
\end{proof}

\section{Proof of Theorem \ref{sing_val_feat}.}\label{sec_proof_sing_val_feat}

We start with the following lemma, whose proof is a simple computation, see
E.3 of \cite{nguyen2021tight}.

\begin{lemma}\label{centfeat_est}
Let $\widetilde{F}_k = F_k - \mathbb{E}_{X}[F_k]$ denote the centred features. 
Let
\begin{equation*}
    \mu = 
\mathbb{E}_{x \sim \mathcal{P}}[f_k(x)] \in \R^{n_k} \text{ and }
\Lambda = Diag(F_k\mu - \vert\vert \mu\vert\vert^2_21_N)
\end{equation*}
where 
$1_N \in \R^N$ is the column vector of $1's$. Then for $1 \leq k \leq L-1$ we have
\begin{equation*}
F_kF_k^T \geq \bigg{(}
\widetilde{F}_k\widetilde{F}_k^T - 
\frac{\Lambda 1_N1_N^T\Lambda}{\vert\vert\mu\vert\vert^2_2}
\bigg{)}
\end{equation*}
where $\geq$ sign is used in the sense of positive semi-definite matrices, meaning
\begin{equation*}
	F_kF_k^T - 
	\bigg{(}
\widetilde{F}_k\widetilde{F}_k^T - 
\frac{\Lambda 1_N1_N^T\Lambda}{\vert\vert\mu\vert\vert^2_2}
\bigg{)} \geq 0
\end{equation*}
i.e. the difference is positive semi-definite.
\end{lemma}

\begin{proof}[\textbf{Proof of Theorem \ref{sing_val_feat}}]

The assumption $n_k \geq N$ implies 
$\sigma_{\min}(F_k)^2 = \lambda_{\min}(F_kF_k^T)$. The proof will proceed by bounding 
$\lambda_{\min}(F_kF_k^T)$.

By lemma \ref{centfeat_est}, in order to bound 
$\lambda_{\min}(F_kF_k^T)$ is suffices to bound 
$\lambda_{\min}(\widetilde{F}_k\widetilde{F}_k^T - 
\frac{\Lambda 1_N1_N^T\Lambda}{\vert\vert\mu\vert\vert^2_2})$. The proof will focus on bounding this latter quantity. 

By Weyl's inequality we have
\begin{equation}\label{weyl_est}
	\lambda_{\min}\bigg{(}\widetilde{F}_k\widetilde{F}_k^T - 
\frac{\Lambda 1_N1_N^T\Lambda}{\vert\vert\mu\vert\vert^2_2}\bigg{)} 
\geq 
\lambda_{\min}(\widetilde{F}_k\widetilde{F}_k^T) - 
\lambda_{\max}(\frac{\Lambda 1_N1_N^T\Lambda}{\vert\vert\mu\vert\vert^2_2})).
\end{equation}
We start by bounding $\lambda_{\min}(\widetilde{F}_k\widetilde{F}_k^T)$. 

By the Gershgorin circle Theorem we have
\begin{align}
\lambda_{\min}(\widetilde{F}_k\widetilde{F}_k^T) &\geq 
\min_{i \in [N]}\vert\vert (\widetilde{F}_k)_{i:}\vert\vert^2_2 - N
max_{i\neq j}\vert\langle (\widetilde{F}_k)_{i:}, (\widetilde{F}_k)_{j:}
\rangle\vert \label{gersh1} \\
\lambda_{\min}(\widetilde{F}_k\widetilde{F}_k^T) &\leq 
\max_{i \in [N]}\vert\vert (\widetilde{F}_k)_{i:}\vert\vert^2_2 + N
max_{i\neq j}\vert\langle (\widetilde{F}_k)_{i:}, (\widetilde{F}_k)_{j:}
\rangle\vert. \label{gersh2}
\end{align}
By lemma \ref{c3}, we have for all $i \in [N]$ that
\begin{equation}
\vert\vert f_k(x_i) - \mathbb{E}_x[f_k(x)]\vert\vert_2^2 = 
\Theta\bigg{(} 
\sqrt{n_0}\beta_kn_k\prod_{l=1}^{k-1}\sqrt{\beta_l}\sqrt{n_l}
\bigg{)}\label{gersh_est1}
\end{equation}
w.p. $\geq 
1 - Nexp\bigg{(}-\Omega \bigg{(} 
\frac{\min_{l\in [0,k]}n_l}{\prod_{l=1}^{k-1}log(n_l)}
\bigg{)}
\bigg{)}
- \sum_{l=1}^kexp(-\Omega(n_l))$ over $(W_l)_{l=1}^k$ and $x$.

The goal is to find a bound for 
$\vert\langle (\widetilde{F}_k)_{i:}, (\widetilde{F}_k)_{j:}
\rangle\vert$. By assumption A4 we have that
\begin{equation*}
	\vert\vert f_k(x) - \mathbb{E}_xf_k(x)\vert\vert^2_{Lip} = 
	\mathcal{O}\bigg{(}
	\frac{s^k}{\min_{l \in [0,k]}n_l}\beta_kn_k\prod_{l=1}^{k-1}
	\sqrt{\beta_l}\sqrt{n_l}\prod_{l=1}^{k-1}Log(n_l)
	\bigg{)}
\end{equation*}
w.p. $\geq 1 - \sum_{l=1}^k2exp(-\Omega(n_l))$ over $(W_l)_{l=1}^k$,
where we used the fact that $f_k(x) - \mathbb{E}_xf_k(x)$ and $f_k(x)$ have the same Lipschitz constant.

We are going to condition on the intersection of the above event over 
$(W_l)_{l=1}^k$ and the event defined by \eqref{gersh_est1} over 
$(W_l)_{l=1}^k$ and $x_j$ and derive bounds over $x_i$. 
Since we have conditioned on $x_j$, 
$\vert\langle (\widetilde{F}_k)_{i:}, (\widetilde{F}_k)_{j:}
\rangle\vert$ is a function of $x_i$ for every $i \neq j$. We then have
\begin{align*}
\bigg{\vert}\bigg{\vert} \vert\langle (\widetilde{F}_k)_{i:}, (\widetilde{F}_k)_{j:}
\rangle\vert 
\bigg{\vert}\bigg{\vert}_{Lip} &\leq 
\vert\vert (\widetilde{F}_k)_{j:}\vert\vert^2_2
\vert\vert f_k(x_i) - \mathbb{E}_{x}f_k(x_i)\vert\vert^2_{Lip} \\
&=
\mathcal{O}\bigg{(}
\frac{s^k\sqrt{n_0}}{\min_{l \in [0,k]}n_l}\beta_k^2n_k^2\prod_{l=1}^k
\beta_ln_l\prod_{l=1}^{k-1}Log(n_l)
\bigg{)}
\end{align*}
using the above two asymptotic estimates we have conditioned on. Note that the above holds for all $x_i \neq x_j$.

Applying our Lipschitz concentration assumption A3, and taking the union of the above estimate over all $x_i \neq x_j$ we have
\begin{equation*}
\mathbb{P}\bigg{(}
\max_{i \in [N], i\neq j}
\vert\langle (\widetilde{F}_k)_{i:}, (\widetilde{F}_k)_{j:}
\rangle\vert \geq t
\bigg{)}
\leq
(N-1)exp\left(
-\frac{t^2}{\mathcal{O}
\bigg{(}
\frac{s^k\sqrt{n_0}}{\min_{l \in [0,k]}n_l}\beta_k^2n_k^2\prod_{l=1}^k
\beta_ln_l\prod_{l=1}^{k-1}Log(n_l
\bigg{)}
}
\right).
\end{equation*}
Choosing $t = N^{-1}\sqrt{n_0}\beta_kn_k\prod_{l=1}^{k-1}\sqrt{\beta_l}
\sqrt{n_l}$. We have
\begin{equation*}
	Nmax_{i\neq j}\vert\langle (\widetilde{F}_k)_{i:}, (\widetilde{F}_k)_{j:}
\rangle\vert \leq 
\beta_kn_k\prod_{l=1}^{k-1}\sqrt{\beta_l}\sqrt{n_l}
\end{equation*}
w.p. $\geq 1 - (N-1)exp\left(-
\Omega\left(
\frac{\min_{l \in [0,k]}n_l}{s^k\prod_{l=1}^{k-1}Log(n_l)}
\right)
\right) - \sum_{l=1}^k2exp(-\Omega(n_l))$.
The above estimate is true for every $x_j$. Therefore we can take a union of the bounds for each $j \in [N]$ to obtain
\begin{equation*}
	Nmax_{i\neq j}\vert\langle (\widetilde{F}_k)_{i:}, (\widetilde{F}_k)_{j:}
\rangle\vert = 
o\left(
\beta_kn_k\prod_{l=1}^{k-1}\sqrt{\beta_l}\sqrt{n_l}
\right)
\end{equation*}
w.p. $\geq 1 - N(N-1)exp\left(-
\Omega\left(
\frac{\min_{l \in [0,k]}n_l}{s^k\prod_{l=1}^{k-1}Log(n_l)}
\right)
\right) - N\sum_{l=1}^k2exp(-\Omega(n_l))$.

Thus we obtain 
\begin{equation}\label{gersh_est2}
\lambda_{\min}(\widetilde{F_k}\widetilde{F_k}^T) = \Theta\left(
\beta_kn_k\prod_{l=1}^{k-1}\sqrt{\beta_l}\sqrt{n_l}
\right)
\end{equation}
w.p. $\geq 1 - N(N-1)exp\left(-
\Omega\left(
\frac{\min_{l \in [0,k]}n_l}{s^k\prod_{l=1}^{k-1}Log(n_l)}
\right)
\right) - N\sum_{l=1}^k2exp(-\Omega(n_l))$. This bounds the first term on the right hand side of \eqref{weyl_est}. We move on to bounding the second term on the right hand side of \eqref{weyl_est}.

We want to bound the maximum eigenvalue of the quantity 
$\frac{\Lambda 1_N1_N^T\Lambda}{\vert\vert \mu\vert\vert^2_2}$. The maximum eigenvalue is the operator norm, therefore we will obtain an estimate for the operator norm. As a start we have the simple estimate
\begin{equation*}
\bigg{\vert}\bigg{\vert}
\frac{\Lambda 1_N1_N^T\Lambda}{\vert\vert \mu\vert\vert^2_2}
\bigg{\vert}\bigg{\vert}_{op}
\leq 
\bigg{\vert}\bigg{\vert}
\frac{\Lambda}{\vert\vert \mu\vert\vert_2}
\bigg{\vert}\bigg{\vert}_{op}^2.
\end{equation*}
We define an auxiliary random variable $g : \R^d \rightarrow \R$ by 
$g(x) = \langle f_k(x), \mu\rangle$. Note that 
$\Lambda_{ii} = g(x_i) - \mathbb{E}_x[g(x)]$ and that 
$\vert\vert g\vert\vert_{Lip}^2 \leq \vert\vert\mu\vert\vert_2^2
\vert\vert f_k\vert\vert_{Lip}^2$. Therefore, applying Liptshitz concentration, we get 
\begin{equation*}
\mathbb{P}\left(\vert \Lambda_{ii}\vert \geq t\right) \leq 
exp\left(
- \frac{t^2}{2\vert\vert\mu\vert\vert_2^2\vert\vert f_k\vert\vert^2_{Lip}}
\right).
\end{equation*}
From lemma \ref{c.2}, we have 
\begin{equation}\label{mu_est}
	\vert\vert\mu\vert\vert_2^2 = \Theta\left(
	\sqrt{n_0}\beta_kn_k\prod_{l=1}^{k-1}\sqrt{\beta_l}\sqrt{n_l}
	\right).
\end{equation}
Furthermore, our assumption on the Lipschitz constant gives the estimate
\begin{equation}\label{lip_assump_mu}
	\vert\vert f_k(x)\vert\vert^2_{Lip} = 
	\mathcal{O}\bigg{(}
	\frac{s^k}{\min_{l \in [0,k]}n_l}\beta_kn_k\prod_{l=1}^{k-1}
	\sqrt{\beta_l}\sqrt{n_l}\prod_{l=1}^{k-1}Log(n_l)
	\bigg{)}.
\end{equation}
If we take $t = \frac{1}{N}\vert\vert\mu\vert\vert_2^2$, and take a union bound over all the samples $\{x_i\}$ and the events defined by \eqref{mu_est}, 
\eqref{lip_assump_mu}, we get the estimate
\begin{equation}
\bigg{\vert}\bigg{\vert}
\frac{\Lambda}{\vert\vert \mu\vert\vert_2}
\bigg{\vert}\bigg{\vert}_{op}^2 = \mathcal{O}\left(
\frac{1}{N^2}\sqrt{n_0}\beta_kn_k\prod_{l=1}^{k-1}\sqrt{\beta_l}\sqrt{n_l}
\right)
\end{equation}
w.p. 
$\geq 1 - Nexp\left(-
\Omega\left(
\frac{\min_{l \in [0,k]}n_l}{s^kN^2\prod_{l=1}^{k-1}Log(n_l)}
\right)
\right) - \sum_{l=1}^k2exp(-\Omega(n_l))$.

Putting the estimate for $\lambda_{\min}
\left(\widetilde{F_k}\widetilde{F_k}^T\right)$ and 
$\bigg{\vert}\bigg{\vert}
\frac{\Lambda 1_N1_N^T\Lambda}{\vert\vert \mu\vert\vert^2_2}
\bigg{\vert}\bigg{\vert}_{op}$ together we obtain
\begin{equation*}
\lambda_{\min}\left(
F_kF_k^T
\right)
= 
\Theta\left(
	\sqrt{n_0}\beta_kn_k\prod_{l=1}^{k-1}\sqrt{\beta_l}\sqrt{n_l}
	\right)
\end{equation*}
w.p.  $\geq 1 - N(N-1)exp\left(-
\Omega\left(
\frac{\min_{l \in [0,k]}n_l}{s^kN^2\prod_{l=1}^{k-1}Log(n_l)}
\right)
\right) - N\sum_{l=1}^k2exp(-\Omega(n_l))$.
\end{proof}

\end{document}